\documentclass[11pt]{article}

\let\counterwithin\relax

\usepackage{amsmath, amsfonts, amssymb, amsthm, bm, graphicx, mathtools, enumerate,multirow}
\usepackage{natbib}
\usepackage[CJKbookmarks=true,
bookmarksnumbered=true,
bookmarksopen=true,
colorlinks=true,
citecolor=blue,
linkcolor=blue,
anchorcolor=blue,
urlcolor=blue]{hyperref}
\usepackage[usenames]{color}
\usepackage[letterpaper, left=1.2truein, right=1.2truein, top = 1.2truein,
bottom = 1.2truein]{geometry}
\usepackage[ruled, lined, commentsnumbered]{algorithm2e}
\usepackage{prettyref,soul}

\usepackage{apptools} %Ya
\usepackage{chngcntr} % define \counterwithin
\AtAppendix{\counterwithin{lemma}{section}} %Ya
\usepackage{tikz}
\usepackage[font=small,labelfont=bf]{caption}
\usepackage[section]{placeins}

\newtheorem{lemma}{Lemma}
\newtheorem{proposition}{Proposition}
\newtheorem{thm}{Theorem}
\newtheorem{definition}{Definition}
\newtheorem{corollary}{Corollary}
\newtheorem{assumption}{Assumption}
\theoremstyle{definition}

\newrefformat{eq}{(\ref{#1})}
\newrefformat{chap}{Chapter~\ref{#1}}
\newrefformat{sec}{Section~\ref{#1}}
\newrefformat{algo}{Algorithm~\ref{#1}}
\newrefformat{fig}{Fig.~\ref{#1}}
\newrefformat{tab}{Table~\ref{#1}}
\newrefformat{rmk}{Remark~\ref{#1}}
\newrefformat{clm}{Claim~\ref{#1}}
\newrefformat{def}{Definition~\ref{#1}}
\newrefformat{cor}{Corollary~\ref{#1}}
\newrefformat{lmm}{Lemma~\ref{#1}}
\newrefformat{lemma}{Lemma~\ref{#1}}
\newrefformat{prop}{Proposition~\ref{#1}}
\newrefformat{app}{Appendix~\ref{#1}}
\newrefformat{ex}{Example~\ref{#1}}
\newrefformat{cond}{Condition~\ref{#1}}

\def\vec{\mathrm{vec}} % the symbol Var for covariance used the sans serif letter

\newcommand{\argmin}{\ensuremath{\operatornamewithlimits{arg\,min}}}
\newcommand{\rank}{\mathrm{rank}}
\newtheorem{example}{Example}

\newcommand{\tp}{\intercal}

\newcommand{\bigO}{\ensuremath{\mathop{}\mathopen{}\mathcal{O}\mathopen{}}}

\newcommand{\bigOp}{\bigO_\mathrm{p}}

\newcommand{\R}{\mathbb{R}}

\newcommand{\bbeta}{\bm{\beta}}
\newcommand{\btheta}{\bm{\theta}}

\newcommand{\supp}[1]{#1}

\usepackage[T1]{fontenc}
\usepackage[utf8]{inputenc}
\usepackage{authblk}
\title{CP degeneracy in Tensor Regression}

\author[1]{Ya Zhou}
\author[2]{Raymond K. W. Wong}
\author[1]{Kejun He}
\affil[1]{Institute of Statistics and Big Data, Renmin University of China}
\affil[2]{Department of Statistics, Texas A\&M University,
College Station,
USA}

\date{}

\begin{document}
\maketitle

\begin{abstract}
  Tensor linear regression is an important and useful tool for analyzing tensor data. To deal with high dimensionality, CANDECOMP/PARAFAC (CP) low-rank constraints are often imposed on the coefficient tensor parameter in the (penalized) $M$-estimation. However, we show that the corresponding optimization may not be attainable, and when this happens, the estimator is not well-defined. This is closely related to a phenomenon, called CP degeneracy, in low-rank tensor approximation problems. In this article, we provide useful results of CP degeneracy in tensor regression problems. In addition, we provide a general penalized strategy as a solution to overcome CP degeneracy. The asymptotic properties of the resulting estimation are also studied. Numerical experiments are conducted to illustrate our findings.
\end{abstract}
%\keywords{Nonlinear regression; Tensor low rank; Polynomial splines; Region selection; Elastic-net penalization.}

%! TEX root = CPde.tex
\section{Introduction}

Tensor linear regression \citep{raskutti2019convex, guo2011tensor, Zhou-Li-Zhu13, hoff2015multilinear, sun2017store, lock2018tensor,miranda2018tprm, chen2019non} has recently received an increasing amount of attention among the statistical researchers. In most existing work, low-rank constraints are imposed to overcome the curse of dimensionality \citep{chen2019non, guo2011tensor, Zhou-Li-Zhu13, lock2018tensor, li2018tucker, raskutti2019convex, zhang2020islet}. One popular low-rank modeling is based on CANDECOMP/PARAFAC (CP) decomposition \citep{harshman1970foundations}. However, the corresponding low-rank space may not be compact \citep{de2008tensor}. This raises questions about the well-posedness of the corresponding $M$-estimators, as the solution of their optimizations may not be attainable.
Similar issues, known as CP degeneracy, have been discussed in the context of low-rank \textit{approximation} problems \citep{de2008tensor, krijnen2008non}, which can be treated as a special case of tensor linear regression (see Example \ref{exam:CPD} in Section \ref{ssec:existDegn}). 
However, such fundamental and important issues are seldomly discussed in the statistics literature.
In this article, we aim to provide corresponding understanding and discussion in the context of tensor linear regression, from a statistical perspective.
In particular, we will discuss the existence and characteristics of the CP degeneracy in tensor linear regression. We will also provide a solution to CP degeneracy and investigate the properties of the resulting estimation procedure.

The contribution of this article is threefold. 
First, we construct a set of examples to illustrate the possibility of CP degeneracy in tensor regression settings.
Second, we theoretically derive some important characteristics of CP degeneracy. A crucial characteristic of CP degeneracy is that the magnitude of the CP parameter will diverge, if an iterative algorithm is adopted to obtain an $M$-estimator of the CP parameter.
Third, we propose a general penalized method with the theoretical guarantee to surmount the barrier of degeneracy.
In particular, we provide an asymptotic analysis of the penalized estimation in the high-dimensional regime. To the best of our knowledge, this is the first analysis that does not require the assumption of the well-defined low-rank approximation. That is, in our analysis, the underlying low-rank approximation of the true tensor coefficient \textit{does not} have to exist.

The rest of this article is organized as follows. Section \ref{sec:pre} provides the background of CP decomposition and degeneracy, as well as tensor linear regression. Section \ref{sec:theory} discusses the existence of CP degeneracy, its characteristics and a strategy to overcome it. Section \ref{sec:simulation} provides corresponding numerical experiments. \supp{Technical details are deferred to the Section \ref{supp}}.

\section{Preliminary}\label{sec:pre}
\subsection{Tensor basics and CP degeneracy}
In this article, we focus on tensors that are real-valued multidimensional array.
For a tensor $\mathbf A = (A_{i_1,\ldots,i_D}) \in \mathbb{R}^{p_1 \times \cdots \times p_D}$ with $p_d \in \mathbb{Z}^+$ for $d=1,\ldots,D$, there are $D$ modes where $D$ is often referred to as the order of the tensor.
Vectors and matrices can be regarded as first-order and second-order tensors respectively.
Often, tensors of order 3 or above are referred to as higher-order tensors.
It is now well-known that higher-order tensors behave very differently from its lower-order counterparts, e.g., matrices.
Indeed, many tensor problems are NP-hard \citep{Hillar-Lim13}.
Also, many matrix-related concepts cannot be directly generalized to higher-order tensors.
The CP degeneracy problem (with details given below) --- the major subject of this article --- is an example.
To avoid clutter, we will simply refer to a higher-order tensor as a tensor, and use vector and matrix specifically for first-order and second-order tensor.

One of the most important matrix tools is singular value decomposition (SVD). 
As for tensors, there are two commonly used tensor decompositions, CANDECOMP/PARAFAC (CP) decomposition \citep{harshman1970foundations} and Tucker decomposition \citep{tucker1966some}, both of which can be regarded as higher-order generalizations of  SVD. 
The major interest of this article centers around the CP decomposition.
More specifically, the CP decomposition of a tensor $\mathbf A \in \mathbb{R}^{p_1 \times \cdots \times p_D}$ is written as 
\begin{equation}\label{eqn:CPdecom}
	\mathbf{A}  =  \sum_{r=1}^{R} \bbeta_{1,r} \circ \cdots \circ \bbeta_{D, r},
\end{equation}
where $\bbeta_{d,r} \in \mathbb{R}^{p_d}$, $d=1,\ldots,D$, and $\circ$ represents the outer product.
We refer to the entries of $\bbeta_{d,r}$'s as CP parameters.
The rank of a tensor $\mathbf A$, denoted by $\rank(\mathbf A)$, is the smallest $R$ such that the CP decomposition \eqref{eqn:CPdecom} holds.
In other words, it is the minimal number of rank-1 tensors (outer products of $D$ vectors) that can represent the tensor.
Despite its high similarity to the definition of matrix rank, tensor rank possesses surprisingly different properties.
First, the rank of tensor may depend on whether the entries are real-valued or complex-valued (i.e., the fields).
To date, this seems to be relatively less interesting to the statistical community, as most of the statistical work related to tensors focus on real-valued tensors.
Second, the rank determination is NP-hard \citep{johan1990tensor}.
Third, the space of low-rank tensors may not be closed. More precisely, it has been shown by \citet{de2008tensor} that the space
\[
	\{\mathbf A: \rank(\mathbf A) \le R, \mathbf A \in \mathbb{R}^{p_1 \times \cdots \times p_D} \}
\]
is not closed regardless of the choice of norm, for  $R =2, \ldots, \min\{p_1,\ldots,p_D\}$ and $D\ge 3$.
As a result, the best low-rank approximation for a higher order tensor may not exist.
More specifically, the best rank-$R$ approximation is defined as a solution to the following problem
\begin{equation}\label{eqn:lowrank:appro}
	\argmin_{\rank(\mathbf A) \le R} \Vert \mathbf G - \mathbf A \Vert_{F}.
\end{equation}
Here $\Vert \cdot \Vert_{F}$ denotes the Frobenius norm, also called Hilbert-Schmidt norm, which is defined as 
$\Vert \mathbf A \Vert_{F}  %= \sqrt{\langle \mathbf A, \mathbf A \rangle}  
= \sqrt{\sum_{i_1,\ldots,i_D} A_{i_1,\ldots,i_D}^2}$.
The non-closedness of the feasible set in \eqref{eqn:lowrank:appro} has a significant impact on the existence of a solution.
When the solution of \eqref{eqn:lowrank:appro} does not exist, it is generally referred to as a CP degeneracy problem \citep{kolda2009tensor}.
Note that the degeneracy problem never happens for matrices (i.e., $D=2$), and a best rank-$R$ approximation for a matrix can be obtained efficiently by the $R$ leading factors of the SVD \citep{eckart1936approximation}.
Indeed, the possible non-existence of a best low-rank approximation has a significant influence on the well-posedness of low-rank tensor modeling in statistical problems.
The goal of this article is to investigate this in the low-rank tensor regression problems, and provide solutions to avoid the CP degeneracy in the regression settings.

To better understand this non-existence of best rank-$R$ approximation for a higher-order tensor, \textit{border rank} was introduced in \cite{1980Approximate}. 
%\kejun{Here, the initial reference of ``boarder rank'' is needed. However this reference looks wired??}\ya{updated} 
The border rank of a tenor $\mathbf A$ is defined as
\begin{equation*}
	\begin{aligned}
		\rank_B(\mathbf A) := \min\{R: \text{for any } \epsilon > 0, &\text{ there exists a tensor} \  \mathbf C   \ \\& \quad \quad  \text{such that}   \Vert \mathbf C \Vert_{F} < \epsilon    \ \text{and} \ \rank(\mathbf A + \mathbf C)= R \}.
	\end{aligned}
\end{equation*}
When $\rank_{\rm B}(\mathbf A) < \rank(\mathbf A)$, $\mathbf A$ does not have a best rank-$R$ approximation for $\rank_{\rm B}(\mathbf A) \leq R < \rank(\mathbf A)$.
Many examples of tensors with border rank smaller than the tensor rank
can be found in the literature \citep[e.g.,][]{1980Approximate, paatero2000construction, stegeman2006degeneracy, stegeman2007degeneracy, de2008tensor}.
For example, let
\[
	\mathbf G  = \mathbf v_1 \circ \mathbf w_2 \circ \mathbf w_3 + \mathbf w_1 \circ \mathbf v_2 \circ \mathbf w_3+ \mathbf w_1 \circ \mathbf w_2 \circ \mathbf v_3,
\]
where the pair $\mathbf w_d,\mathbf v_d \in \R^{p_d}$ are linearly independent for each $d=1,2,3$.
As shown in \citet{de2008tensor}, $\rank_B(\mathbf G)=2$ while $\rank(\mathbf G)=3$.
One can construct the rank-2 tensor sequence
$\{\mathbf G_{\gamma}, \gamma = 1,2,\ldots \}$:
\begin{equation}\label{eqn:def:A_alpha}
	\mathbf G_{\gamma} = \gamma \bigg(\mathbf w_1 + \frac{1}{\gamma} \mathbf v_1 \bigg) \circ \bigg(\mathbf w_2 + \frac{1}{\gamma} \mathbf v_2 \bigg) \circ \bigg(\mathbf w_3 + \frac{1}{\gamma} \mathbf v_3 \bigg ) - \gamma \mathbf w_1  \circ \mathbf w_2  \circ \mathbf w_3.
\end{equation}
%and show that  $\rank(\mathbf G_{\gamma})=2$ 
One can see that $\Vert \mathbf G_{\gamma} - \mathbf G\Vert_{F} \to 0 $ as $\gamma \to \infty$.
In other words, this sequence of rank-2 tensors converges to a rank-3 tensor.
Let us suppose (on the contrary that) there exists a solution $\mathbf G_0$ to the best rank-2 approximation problem \eqref{eqn:lowrank:appro}. %, i.e.,
It follows from the assumption that $\rank(\mathbf G_0) \le 2$ and $\Vert \mathbf G_0 - \mathbf G\Vert_{F} >0$. However, by \eqref{eqn:def:A_alpha}, we can always find a $\gamma$ such that $\Vert \mathbf G_{\gamma} - \mathbf G\Vert_{F} < \Vert \mathbf G_0 - \mathbf G\Vert_{F} $, which leads to a contradiction. Therefore, %\eqref{eqn:CP_de} 
\eqref{eqn:lowrank:appro} does not have a solution in this example. 
One interesting phenomenon in this example is that, as $\mathbf G_{\gamma}$ converges to $\mathbf G$, some CP parameters of  $\mathbf G_{\gamma}$ diverge. 
Indeed, this example is not a special case. \citet{krijnen2008non} shows that if the low-rank approximation does not have a solution, then some CP parameters of the sequence whose objective value converging to the infimum will diverge. 
It is worth noting that for matrices (i.e. $D =2$), rank and board rank always coincide. For higher-order tensors (i.e. $D \geq 3$), there exists a positive volume of the tensors which do not have a best low-rank approximation \citep{de2008tensor}.
Therefore, it may result in numerical problems in practice, if one blindly computes the low-rank approximation.

In recent years, the research of tensor regression has gained popularity.
Due to its empirical successes, the low-rank assumption is among the most popular modeling strategies used to overcome the ``large-$p$-small-$n$'' problem of the underlying tensor parameter estimations 
\citep[e.g.,][]{Zhou-Li-Zhu13,rabusseau2016low, guhaniyogi2017bayesian,hao2019sparse,lock2018tensor,raskutti2019convex}. 
The low-rank assumption is often just a working assumption, and may only hold approximately.
In this case, the well-posedness of the low-rank modeling hinges on the existence of low-rank \textit{approximation}.
However, the related CP degeneracy problem in tensor regression is rarely mentioned.
To the best of the authors' knowledge, there is virtually no work for understanding this issue in the context of tensor regression, which we hope to address in this article.
We will also provide strategies to avoid the CP degeneracy in this statistical setup.

\subsection{Tensor linear regression}
We briefly review the tensor linear regression model \citep[e.g.,][]{guo2011tensor, Zhou-Li-Zhu13, hoff2015multilinear, suzuki2015convergence, yu2016learning, sun2017store, li2017parsimonious,guhaniyogi2017bayesian, li2018tucker, lock2018tensor, kang2018scalar, raskutti2019convex, chen2019non, zhang2020islet}.
Suppose we have a data set $\{(y_i, \mathbf X_i) : i=1,\dots, n\}$,
where $y_i \in \mathbb{R}$ is a response variable and $\mathbf X_i \in \mathbb{R}^{p_1 \times \dots \times p_D}$ is a tensor covariate.
The tensor linear model assumes 
\begin{equation}\label{eqn:TLR_model_A}
	y_i = \langle \mathbf A_0, \mathbf X_i \rangle+ \epsilon_i,
\end{equation}
where $\langle \cdot, \cdot \rangle$ is the entry-wise inner product, $\mathbf A_0\in \mathbb{R}^{p_1\times \dots\times p_D}$ is an unknown fixed coefficient tensor, and $\epsilon_i\in\mathbb{R}$ is a random error of mean zero. The commonly used squared loss function leads us to
\begin{equation}\label{eqn:def:F}
	\rm F (\mathbf A ):= \sum_{i=1}^n ( y_i -  \langle \mathbf A, \mathbf X_i \rangle   )^2.
\end{equation}
In some applications, it is reasonable to assume $\mathbf A_0$ admits a low-rank structure \citep{guo2011tensor, Zhou-Li-Zhu13}.
One can directly impose the rank restriction $\rank(\mathbf A) \le R$ in \eqref{eqn:def:F} and estimate the coefficient tensor by solving
\begin{equation}\label{eqn:opt_openset}
	\argmin_{\rank(\mathbf A) \le  R}  \rm F(\mathbf A).
\end{equation}

Another way to integrate the rank-$R$ assumption is to reparametrize $\mathbf A$ using the CP decomposition form \eqref{eqn:CPdecom}, i.e., 
\[
\mathbf{A} =  \sum_{r=1}^{R} \bbeta_{1,r} \circ \cdots \circ \bbeta_{D, r}.
\]
%Denoting $\btheta$ collectively as
Let $\btheta$ and $\mathbf B_{d}$  be the collections of all CP parameters and those in the $d$-th dimension respectively. In other words, we write
\begin{equation}\label{eqn:def_theta}
	\btheta = % \vec\{ \vec(\mathbf B_1), \cdots, \vec(\mathbf B_D)\},
	( \vec(\mathbf B_{1})^\tp, \cdots, \vec(\mathbf B_{D})^\tp )
	\quad \mbox{and } \quad \mathbf B_{d} = (\bbeta_{d,1}, \cdots, \bbeta_{d,R} ).
\end{equation} 
The squared loss w.r.t. $\btheta$ can then be written as
\begin{equation}\label{eqn:def:f0}
	f(\btheta) := \sum_{i=1}^n \bigg( y_i-  \bigg \langle  \sum_{r=1}^{R} \bbeta_{1,r} \circ \cdots \circ \bbeta_{D, r}, \mathbf X_i \bigg \rangle  \bigg)^2
\end{equation}
and the corresponding estimator is obtained by solving
\begin{equation}\label{eqn:def:f}
	\argmin_{\btheta} f(\btheta).
\end{equation} 
From \eqref{eqn:opt_openset} and \eqref{eqn:def:f}, we can see there are two ways to parametrize the low-rank coefficient tensors: (i) the tensor $\mathbf{A}$ itself with the rank restriction; (ii) the CP parameters $\btheta$ defined as in \eqref{eqn:def_theta}. 
Without further penalization, these two parametrizations are equivalent, in the sense that if one of the corresponding optimizations cannot attain the infimum, the other one cannot either. 
However, as will be seen in Subsections \ref{ssec:existDegn} and \ref{ssec:charDegn}, using CP parameters $\btheta$ helps characterizing the problem of CP degeneracy. 
Moreover, we will show (in Subsections \ref{ssec:strategy} and \ref{sec:ridge}
) that adding a correct form of penalization on the CP parameters $\btheta$ will solve the degeneracy issue, but not necessarily true when the penalization is directly applied to the coefficient tensor.
Although we mainly focus on the least squares problem in this article, the results in the following sections can be generalized to other commonly used convex objective functions.

\section{Examples, theory and solution}\label{sec:theory}

\subsection{Existence of degeneracy}\label{ssec:existDegn}	
Obviously, the first and the most natural question to ask is whether degeneracy would occur in tensor linear regression. Compared with the best low-rank approximation problem \eqref{eqn:lowrank:appro}, tensor linear regression involves an additional component, i.e., the predictors. 
Whether the degeneracy occurs is also affected by the value of the predictors.
Indeed, it is easy to see that
the classical degeneracy in \eqref{eqn:lowrank:appro} is a special case
of the degeneracy in the regression setting, as follows.
\begin{example}\label{exam:CPD}
	Consider the noiseless case in tensor linear regression, i.e., $\epsilon_i=0$ in \eqref{eqn:TLR_model_A}. Suppose $n=\prod_d p_d$ and the set of design matrices $\{\mathbf{X}_i\}$
	is chosen as
	\begin{equation*}\label{eqn:example:standardbasis}
		\{\mathbf{X}_i\} = \{\mathbf{e}_{1,i_1} \circ \mathbf{e}_{2,i_2} \cdots \circ \mathbf{e}_{D,i_D}: i_d\in\{1,\dots,p_d\}, d=1,\dots, D\},
	\end{equation*}
	%	\kejun{where to cite the equation number?}\ya{updated}
	where $\mathbf{e}_{d,i_d}$ is the standard basis of Euclidean space $ \mathbb{R}^{p_d}$ such that its $i_d$-th component equals 1 and the rest are 0. If the true coefficient tensor $\mathbf{A}_0$ in \eqref{eqn:TLR_model_A} does not have a best rank-$R$ approximation, %$R$ in \eqref{eqn:opt_openset} satisfies $\rm{rank}_{\rm b} (\mathbf A_0)\le R  < \rm{rank} (\mathbf A_0)$, 
	then the corresponding least squares problem \eqref{eqn:opt_openset}  %and \eqref{eqn:def:f} 
	does not have a solution.
\end{example}
Example \ref{exam:CPD} shows that when $\{\mathbf{X}_i\}$ are the full standard basis of $\mathbb{R}^{p_1 \times \cdots \times p_D}$, the optimization \eqref{eqn:opt_openset} can be written as %the low-rank approximation form
\begin{equation}\label{eqn:exam:CPDA0}
	\argmin_{\text{rank}(\mathbf{A})\le R }\Vert \mathbf A_0 - \mathbf{A} \Vert_F,
\end{equation}
which recovers the low-rank approximation form as in \eqref{eqn:lowrank:appro}. If $\mathbf{A}_0$ does not have the  best rank-$R$ approximation, then \eqref{eqn:exam:CPDA0} has no solution and thus the original least squares problem does not have a solution. Although we only consider the noiseless case in Example \ref{exam:CPD}, it can also be shown that the degeneracy could happen with a positive probability for the noisy cases; see Theorem 8.4 of \citet{de2008tensor} for the details.

Next, we will give more interesting examples of the degeneracy in unpenalized tensor linear regressions. We will also consider the penalized regressions with, e.g., the ridge penalty, in Section \ref{sec:ridge}. For simplicity, we denote 
\begin{equation}\label{def:Z}
	\mathbf Z = (\vec( \mathbf X_1) , \ldots, \vec(\mathbf X_n) )^\tp \quad \text{and} \quad  \mathbf y = (y_1, \ldots, y_n)^\tp,
\end{equation}
where $\vec(\cdot)$ is the vectorization operator. 
Using the notations defined in \eqref{def:Z}, the squared loss function $ \rm F(\mathbf{A})$ in \eqref{eqn:def:F} can be rewritten as 
\[
	\rm F ( \mathbf{A}) = \Vert \mathbf y - \mathbf Z \vec (\mathbf A ) \Vert^2,
\]
%\end{equation}
where $\Vert \cdot \Vert$ is the Euclidean norm.
Without the rank restriction, the classical theory of least squares estimation shows that the set of solutions to minimize 
$\rm F ( \mathbf{A})$ is 
\begin{equation}\label{eqn:example:3:defS}
	\mathcal{S}=\{\mathbf A: \vec (\mathbf A) = (\mathbf Z^\tp \mathbf Z)^+ \mathbf Z^\tp  \mathbf{y} + \{\mathbf I - (\mathbf Z^\tp \mathbf Z)^+ \mathbf Z^\tp \mathbf Z\} \mathbf b , \mathbf b \in \mathbb{R}^{\prod_dp_d}\},
\end{equation}
where $(\cdot)^+$ denotes a generalized inverse for a matrix and $\mathbf I$ is an identity matrix of compatible dimension.

\begin{lemma}\label{lem:divergingset}
	Suppose
	%	\begin{equation}\label{eqn:example:3:defRm}
	\[
	R_{\rm m} = \min \{\mathrm{rank}(\mathbf A): \mathbf A \in \mathcal{S} \}
	%	\end{equation}
	\quad \text{and} \quad 
	%  \[
	R_{\rm b} = \min \{\mathrm{rank}_{\rm B}(\mathbf A): \mathbf A \in \mathcal{S} \},
	\]
	where $ \mathcal S$ is the set of solutions of tensor linear model as defined in \eqref{eqn:example:3:defS}.
	If $R_b \le R < R_m$, then the optimizations \eqref{eqn:opt_openset} and \eqref{eqn:def:f} do not have a solution. 
\end{lemma}

\supp{The proof Lemma \ref{lem:divergingset} can be found in Section \ref{proof_of_lemma1}}. 
Using Lemma \ref{lem:divergingset}, we can give the next two concrete examples (Examples \ref{exam:1} and \ref{exam:2}) that illustrate the existence of degeneracy problems.

\begin{example}\label{exam:1} 
	Suppose one observes data $(\mathbf X_i \in \mathbb{R}^{p_1 \times p_2 \times p_3}, y_i \in \mathbb{R})$, $i=1,\dots,n$, satisfying $\mathbf y = \mathbf Z \vec(\mathbf G_b)$ for some $\mathbf G_b \in \mathbb{R}^{p_1 \times p_2 \times p_3}$, where $\mathbf Z$ and $\mathbf y$ are defined in \eqref{def:Z}. If $\rank(\mathbf Z) =  \prod_{d=1}^3 p_d$ and the tensor $\mathbf G_b$ can be represented as 
	%\begin{equation}\label{eqn:borderA}
	\[
	\mathbf G_b = \mathbf v_1 \circ \mathbf w_2 \circ \mathbf w_3 + \mathbf w_1 \circ \mathbf v_2 \circ \mathbf w_3+ \mathbf w_1 \circ \mathbf w_2 \circ \mathbf v_3,
	%	\end{equation}
	\]
	where $\mathbf w_d,\mathbf v_d \in \R^{p_d}$ are pairs of linearly independent vectors for $d=1,2,3$,
	then there is no solution to \eqref{eqn:opt_openset} with $R=2$.
	To see this, we only need to confirm the setting satisfies the condition of Lemma  \ref{lem:divergingset}.
	By definition, we have $\mathcal{S} = \{ \mathbf G_b \}$.
	By Corollary 5.12 of \citet{de2008tensor}, $\mathrm{rank}(\mathbf G_b) =3$ and  $\mathrm{rank}_B(\mathbf G_b) =2$ . Thus, this setting satisfies the condition of Lemma \ref{lem:divergingset}.
\end{example}

Example \ref{exam:1} illustrates a case when the corresponding design matrix $\mathbf{Z}$ in \eqref{def:Z} has full column rank.
Next, we give a more general example that does not require full column rankness of $\mathbf{Z}$, and so $\mathbf{Z}^\tp\mathbf{Z}$ is not necessarily invertible.
%Even without this, the CP degeneracy can occur for more general design in tensor linear regression (Example \ref{exam:2}).
Note that, when $j = (i_1 -1)(p_2p_3) + (i_2-1)p_3 + i_3$, the $j$-th column of $\mathbf Z$ consists of the predictor values at the $(i_1, i_2, i_3)$-th entry in the covariate tensor $\mathbf{X} \in \mathbb{R}^{p_1 \times p_2 \times p_3}$.
\begin{example}\label{exam:2}
	Suppose one observes data $(\mathbf X_i \in \mathbb{R}^{p_1 \times p_2 \times p_3}, y_i \in \mathbb{R})$, $i=1,\dots,n$. 
	Collect the columns of $\mathbf Z$ that correspond to entries of the predictors with positions $\mathcal{I}:=\{(i_1,i_2,i_3),\, i_d=1, \ldots, \widetilde p_d, \, \widetilde p_d <p_d,\, d=1,2,3 \}$
	to form $\mathbf Z_1 \in \mathbb{R}^{n \times \prod_d  \widetilde{p}_d }$, and the rest of the columns to form $\mathbf{Z}_2 \in \mathbb R^{n \times (\prod_d p_d -  \prod_d  \widetilde p_d )}$. 
	Now we suppose $( \mathbf Z_1 ,\mathbf Z_2 )$ satisfies $\mathbf Z_1^\tp \mathbf Z_1 = \mathbf I_s$,  $\mathbf  Z_1^\tp\mathbf Z_2 = \mathbf 0$, where $\mathbf I_s \in \mathbb{R}^{s \times s} $ is the identity matrix, $s=\widetilde{p}_1\widetilde{p}_2 \widetilde{p}_3$. If $\mathbf y = \mathbf Z_1 \vec(\widetilde{\mathbf G}_b)$ for some $\widetilde{\mathbf G}_b \in \mathbb{R}^{\widetilde{p}_1 \times \widetilde{p}_2 \times \widetilde{p}_3}$ and	$\widetilde{\mathbf{G}}_b$ can be represented as 
	%\begin{equation}\label{def:tildeAb}
	\[
		\widetilde{\mathbf G}_b = \widetilde{\mathbf v}_1 \circ \widetilde{\mathbf w}_2 \circ \widetilde{\mathbf w}_3 + \widetilde{\mathbf w}_1 \circ \widetilde{\mathbf v}_2 \circ \widetilde{\mathbf w}_3+ \widetilde{\mathbf w}_1 \circ \widetilde{\mathbf w}_2 \circ \widetilde{\mathbf v}_3,
	\]
	%\end{equation}
	where $\widetilde{\mathbf w}_d,\widetilde{\mathbf v}_d \in \R^{\widetilde{p}_d}$ is a pair of linearly independent vectors,  $d=1,2,3$, then there is no solution when one uses \eqref{eqn:opt_openset} with $R=2$ to fit the low-rank tensor linear regression.
	\supp{The proof of this result is given in Section \ref{proof_all_exam}.}
\end{example}

After demonstrating the possibility of CP degeneracy in tensor linear regression, it is important to understand this phenomenon so as to avoid it.
The following subsection will provide some helpful characteristics of CP degeneracy.

\subsection{Characterization of degeneracy} \label{ssec:charDegn}
Next, we provide some useful and important characteristics of CP degeneracy in tensor regression.
First, Theorem \ref{thm:diverging} below shows that if the degeneracy occurs,
the CP parameters will diverge. Note that there is a scaling indeterminacy in the CP decomposition, e.g., 
\[
\bbeta_{1,r} \circ \cdots \circ \bbeta_{D, r} =\frac{1}{S} \bbeta_{1,r} \circ \cdots \circ (S\bbeta_{D, r}), \quad \text{for any } \ S \ne 0,
\]
which will affect the magnitude of the CP parameters.
To avoid the effect of scaling indeterminacy, we use
%\begin{equation}
%\label{def:Mtheta}
\[
\mathcal{M}(\btheta) =\sum_{r=1}^R \prod_{d=1}^D \Vert  \bbeta_{dr} \Vert,
\]
%\end{equation}
to quantify the magnitude of CP parameters in \eqref{eqn:def_theta}.

\begin{thm} %
	%\label{thm:diverging}
	\label{thm:diverging}
	Let $\{\btheta_t: \btheta_t =(\vec(\mathbf B_{1,t})^\tp, \ldots, \vec(\mathbf B_{D,t})^\tp), \mathbf B_{d,t} =(\bbeta_{d1,t}, \ldots, \bbeta_{dr,t}) \}_t $ be a sequence satisfying $f(\btheta_t) \to \inf f$ in \eqref{eqn:def:f}.
	If the infimum in the optimization \eqref{eqn:def:f} is not attainable, then 
	%\begin{equation}
	%\Vert \btheta_t \Vert_F \to \infty.
	$\mathcal{M}(\btheta_t) \to \infty$.
	%\end{equation}
\end{thm}

\supp{The proof of Theorem \ref{thm:diverging} can be found in Section \ref{proof:thm:diverging}.} 
From a practical perspective, this theorem can be best linked with a situation when an iterative algorithm is (blindly) adopted to minimize $f$.
Indeed, most existing methods to solve the optimization over a low-CP-rank space are iterative algorithms \citep[e.g.][]{guo2011tensor, tan2012logistic, Zhou-Li-Zhu13,lock2018tensor, zhang2018tensor, hao2019sparse, 2020Broadcasted}.
Let $\btheta_t$ represent the output value of $\btheta$ at the $t$-th iteration of the algorithm.
For such an algorithm, a reasonable expectation is that $f(\btheta_t)\rightarrow \inf f$.
As shown in Theorem \ref{thm:diverging}, a characteristic of CP degeneracy is that $\mathcal{M}(\btheta_t)\rightarrow \infty$.
Practically this indeed can be checked, and will be illustrated in Section \ref{sec:simulation}.
In the case of CP degeneracy, these algorithms may produce overflow problems, if one does not limit the maximum number of iterations.
Note that terminating the algorithm after some number of iterations without convergence is an early stopping strategy, which can be regarded as a form of regularization. 
Therefore, to be precise, the resulting estimator is no longer the one defined by \eqref{eqn:def:f}, and would require a separate analysis of their statistical behavior.

Next, we provide Theorems \ref{thm:nonasymptotic} and \ref{thm:conlinar:asym} which present minimum eigenvalue (upper) bounds on the CP parameters. These bounds indicate the approximate linear dependency in the CP parameters when degeneracy occurs. We use $C$ with or without subscripts to denote a positive constant which may change values from line to line. 

\begin{assumption}
	There exists a constant $S_1$ such that
	\label{assump:deterministic_X}
	\[
	0 < S_1 \le \inf_{\mathbf A \in \mathcal{P}} \frac{1}{n} \bigg \vert  \sum_{i=1}^n \langle \mathbf A, \mathbf X_i \rangle \bigg \vert^2, % \le \sup_{\mathbf A \in \mathcal{P}} \frac{1}{n} \bigg \vert  \sum_{i=1}^n \langle \mathbf A, \mathbf X_i \rangle \bigg \vert^2 
	%\le C_2,
	\]
	where 
	\begin{equation}
		\label{def:P}
		\mathcal{P} =\{\mathbf A, \Vert \mathbf A \Vert_{ F}=1, \text{rank}(\mathbf A) \le R \}.
	\end{equation}
\end{assumption}
Assumption \ref{assump:deterministic_X} can be regarded as a restricted eigenvalue condition with respect to the design $\{\mathbf X_i\}$.

For notational simplicity, we write 
%\begin{equation}\label{def:D}
$\mathbf D_t =(\mathbf d_{1,t}, \ldots, \mathbf d_{R,t})$,
%\end{equation}
where 
\[
\mathbf d_{r,t} = \frac{1}{ \prod_d \Vert  \bbeta_{dr,t} \Vert  } \bbeta_{1r,t}  \otimes \cdots \otimes \bbeta_{Dr,t} \in \mathbb{R}^{\prod p_d}.
\]
In the above, $\otimes$ denotes Kronecker product and $\bbeta_{dr,t} \in \mathbb{R}^{p_d}$ is the component of $\btheta_t$  
defined in Theorem \ref{thm:diverging}. We use $ \lambda_{\min}(\cdot)$ and $ \lambda_{\max}(\cdot)$ to denote the minimum and maximum eigenvalues of a matrix respectively. 
\begin{thm}\label{thm:nonasymptotic}
	Under Assumption \ref{assump:deterministic_X}, there exists a constant $t_0 %(\{\btheta_t\}) \ge 0
	$ depending on the sequence $\{\btheta_t \}$ such that if $t \ge t_0$, then 
	\[
	\lambda_{\min}(\mathbf D_t^\tp \mathbf D_t)  \le \frac{R (C_1 + C_2 \Vert \mathbf y \Vert^2) }{n\mathcal{M}^2(\btheta_t)},
	\]
	and
	\[
	\frac{1}{\prod_r \Vert \bbeta_{dr,t} \Vert^2 }\lambda_{\min}(\mathbf B_{d,t}^\tp \mathbf B_{d,t}) \le \frac{R (C_1 + C_2 \Vert \mathbf y \Vert^2)}{n\mathcal{M}^2(\btheta_t)}, \quad \text{for} \quad  d=1,\ldots, D.
	\]
\end{thm}
\supp{The proof of Theorem \ref{thm:nonasymptotic} is presented in Section \ref{proof:thm:nonasymptotic}}. 
By Theorem \ref{thm:diverging}, $\mathcal{M}(\btheta_t)$ diverges if the degeneracy occurs.
%by Theorem \ref{thm:diverging}.
Therefore, in the degeneracy cases, the ranks of the CP component matrices are nearly rank-deficient. 

We next show in \supp{Lemma \ref{lem:eigen}} that Assumption \ref{assump:deterministic_X} holds with high probability in random design settings (Assumption \ref{assump:predictor}). Consequently, an analogously non-asymptotic result is shown in Theorem \ref{thm:conlinar:asym}. We begin by presenting the definitions of sub-Gaussian random variables and vectors.

\begin{definition}[Sub-Gaussian random variables \citep{vershynin2018high}]
	The sub-Gaussian norm of a random variable $X$, denoted by $\Vert X \Vert_{\psi_2}$, is defined as 
	\begin{equation}\label{def:sub_gaussian_variable}
		\Vert X \Vert_{\psi_2} = \inf \{V>0: \mathbb{E} \exp(X^2/V^2) \le 2 \}.
	\end{equation}
	We say $\Vert X \Vert_{\psi_2} =\infty$ when there is no such positive $V$ satisfying \eqref{def:sub_gaussian_variable}. A random variable $X$ is called sub-Gaussian if $\Vert X \Vert_{\psi_2} < \infty$. 
\end{definition}

\begin{definition}[Sub-Gaussian random vectors \citep{vershynin2018high}]
	A random vector $\mathbf{x}$ in $\mathbb{R}^p$ is called sub-Gaussian if the one-dimensional marginals $\langle \mathbf{x}, \mathbf a \rangle $'s are sub-Gaussian random variables for all $\mathbf a \in \mathbb{R}^p$. The sub-Gaussian norm of $\mathbf{x}$ is defined to be
	%\begin{equation}\label{def:sub_gaussian_vector}
	\[
		\Vert \bm x \Vert_{\psi_2} = \sup_{\Vert \mathbf a \Vert=1} \Vert \langle \mathbf x, \mathbf a \rangle  \Vert_{\psi_2}.%,
	\]
	%\end{equation}
\end{definition}

\begin{assumption}\label{assump:predictor} 
	Assume $\mathbf X_i$ are i.i.d. with $\mathbb{E}\{\vec(\mathbf X_i) \vec( \mathbf X_i)^\tp \} = \Sigma$, which satisfies
	\[
	0 < S_{\min} \le \lambda_{\min} (\Sigma ) \le \lambda_{\max}(\Sigma ) \le S_{\max} < \infty, 
	\]
	and
	\[
		\Vert \vec(\mathbf X_i)^\tp \Sigma^{-1/2} \Vert_{\psi_2} \le \kappa < \infty,
	\]
	where $S_{\min}$, $S_{\max}$ and $\kappa$ are constants. 
\end{assumption}

\begin{definition}[Gaussian width] %\citep{vershynin2018high}
	For any subset $\mathcal{A} \subset \mathbb{R}^{p_1 \times \cdots \times p_D}$, %\kejun{$\subset$? Need to be a space?} \ya{Updated}
	the Gaussian width of $\mathcal{A}$ is defined to be
	%\begin{equation}\label{eqn:def:Gaussianwidth}
	$w(\mathcal{A}) = \mathbb{E}_{\mathbf A \in \mathcal{A}} \langle \vec (\mathbf A), \mathbf x \rangle$,
	%\end{equation}
	where $\mathbf x \sim \mathcal{N}(\mathbf 0, \mathbf I_{\prod_{p_d}})$.
\end{definition}

\begin{thm}\label{thm:conlinar:asym}
	Under Assumption \ref{assump:predictor}, if  $n > C_1 w^2(\mathcal{P})$, two statements in Theorem \ref{thm:nonasymptotic} will hold with probability at least $1 - 2 \exp\{-C_2 w^2(\mathcal{P})\}$, where $\mathcal{P}$ is a low-rank space defined in \eqref{def:P}. %and $w(\mathcal{P})$ is the Gaussian width defined as \eqref{eqn:def:Gaussianwidth}.
\end{thm}
Here we provide an upper bound of $w(\mathcal{P})$ as follows.
\begin{thm}\label{thm:gaussanwidthsbound}
	We have $w(\mathcal{P}) \le C(R^D+R\sum_{d=1}^D p_d)$.
\end{thm}
\supp{The proofs of Theorems \ref{thm:conlinar:asym} and \ref{thm:gaussanwidthsbound} are presented in Sections \ref{proof:thm:conlinar:asym} and \ref{proof:gaussanwidthsbound}, respectively}.
Overall, these characteristics demonstrated in Theorems \ref{thm:diverging}--\ref{thm:conlinar:asym} can be used in practice to detect degeneracy.
In fact, they lead to solutions to prevent degeneracy in both theory and practice. We will discuss such a solution rigorously in the next subsection.

\subsection{Strategy to overcome the problem}\label{ssec:strategy}
As suggested by Theorem \ref{thm:diverging}, to avoid the degeneracy problem of tensor linear regression, we should restrict the magnitude measure $\mathcal{M}(\btheta)$.
A simple solution is to impose penalty:
\begin{equation}\label{eqn:penalized:CPlevel}
	\argmin_{\btheta} f(\btheta) + \lambda  g(\btheta),
\end{equation}
where $\lambda >0$ is the penalty parameter and $g$ is the penalty function.
In this article, we let $g(\btheta)$ be a continuous function and satisfies the following assumption.
\begin{assumption}[Conditions on the penalty function] \label{assump:g}
	The function $g$ satisfies
	\begin{enumerate}
		\item $g(\mathbf{0}) <\infty$;
		\item $g(\btheta_t)\rightarrow \infty$ for any sequence $\{\btheta_t, t=1,2,\dots\}$ such that  $\mathcal{M}(\btheta_t) \to \infty$;
		\item Let $\{\btheta_t, \, t = 1,2\dots \}$ be any sequence such that 
		$\mathcal{M}(\btheta_t) \le S_1$ for all $t$ and some $S_1$. 
		Then there exists another sequence $\{\btheta_t', \, t = 1,2\dots \}$ and a constant $S_2 < \infty$, depending on $S_1$, such that for all $t$, $\mathcal{M}(\btheta_t') \le S_1$, $g(\btheta_t') \le g(\btheta_t)$, and $\Vert \btheta_t' \Vert \le S_2$.
	\end{enumerate}
\end{assumption}

Note that %\eqref{eqn:def_g_lamba}
the above assumption on $g(\cdot)$ is mild and is satisfied with many commonly used penalty functions with respect to $\btheta$. 
For example, LASSO \citep{tibshirani1996regression} $$\sum_{d=1}^D \sum_{r=1}^R \sum_{l_d=1}^{p_d}\vert \bbeta_{d,r.l_d} \vert, $$
the group LASSO \citep{yuan2006model} $\sum_{d=1}^D\sum_{l_d=1}^{p_d} \sqrt{\sum_{r=1}^R \bbeta_{d,r.l_d}^2}$, and the Ridge \citep{hastie2013elements}
\begin{equation}\label{eqn:def:penalty:ridge}
	\sum_{d=1}^D \sum_{r=1}^R \sum_{l_d=1}^{p_d} \bbeta_{d,r.l_d}^2,
\end{equation}
where $\bbeta_{d,r.l_d}$ is $l_d$-th element of $\bbeta_{d,r}$.
In fact, they have been used in previous work \citep{guo2011tensor, Zhou-Li-Zhu13,hao2019sparse}
but without any consideration of the CP degeneracy.
With Assumption \ref{assump:g},
%\kejun{add after Ya's revision}\ya{updated}, 
the infimum of $f(\btheta) + \lambda  g(\btheta)$ is attainable for any positive $\lambda$. 
Formally, the following Corollary \ref{lem:f_lambda} shows that CP degeneracy does not occur in \eqref{eqn:penalized:CPlevel} and \supp{its proof is given in Section \ref{proof:lem:f_lambda}.}

\begin{corollary}\label{lem:f_lambda}
	Suppose Assumption \ref{assump:g} is satisfied. If $\lambda >0$, then the optimization \eqref{eqn:penalized:CPlevel} has a solution.
\end{corollary}

As mentioned before, the low-rank approximation of higher-order tensor may not exist and the least squares estimation of tensor linear regression restricted on a low-rank space may not be attainable. 
Thus, when the degeneracy occurs, the numerical results are unstable over the iterations of any iterative algorithms and the solution of the optimization is not well-defined. 
Borrowing the idea of Corollary \ref{lem:f_lambda}, the following Theorem \ref{thm:tenreg} bypasses the rank-$R$ assumption %of a general  approximation 
and presents the asymptotic properties of the estimation \eqref{eqn:penalized:CPlevel}. 
We note that such an assumption occurs in most existing convergence analyses %of the 
%rate 
for tensor linear regression methods \citep[e.g., ][]{Zhou-Li-Zhu13,suzuki2015convergence,lock2018tensor}.
To the best of our knowledge, ours is the first result that does not require such condition, and works even when the best low-rank assumption does not exist.
Before presenting the theorem, we need some assumption on the observational errors which are intrinsic in many statistical applications.
\begin{assumption}\label{assump:error} 
	The observational errors $\epsilon_i, i=1,\ldots,n$ are i.i.d. mean zero sub-Gaussian random variables with sub-Gaussian norm $S_{\epsilon}<\infty$. 
\end{assumption}
 
Due to Example \ref{exam:CPD}, the low-rank approximation problem can be treated as a special case of tensor linear regression model without observational noise. 
Thus, using the proof of Corollary \ref{lem:f_lambda} also guarantees the existence of restricted low-rank approximation.
In particular, for given $R>0$ and $G>0$, denote the low-rank approximation error $\delta$ as 
$$
\delta^2 = \min_{\stackrel{\rank(\mathbf A) \le R,}{g(\btheta ) \le G}}\Vert \mathbf A  - \mathbf A_0 \Vert_{F}^2.
$$

\begin{thm}
	\label{thm:tenreg}
	Suppose $ \hat{\mathbf A}$ is the coefficient tensor reconstructed from a solution $\hat{\btheta}$ of \eqref{eqn:penalized:CPlevel}.
	Under Assumptions \ref{assump:predictor}--\ref{assump:error}, 
	we have
	\begin{equation}
		\label{thm:finalbound}
		\Vert \hat{\mathbf A} - \mathbf A_0 \Vert_{F}^2 \le 
		\bigOp \bigg(\frac{R^D + R\sum_{d=1}^D p_d }{n} \bigg) + \bigOp(\delta^2) +  \bigOp \bigg(\lambda \frac{G}{n} \bigg).
		%C_1 \frac{  \big[ \int_0^2 \{ \log N(\xi,\mathcal{P}, l_2)\}^{1/2} \mathrm{d}\xi \big ] ^2}{n} + C_2 v+ \lambda C_3 \frac{\delta_0}{n},
	\end{equation}
\end{thm}

\supp{The proof of Theorem \ref{thm:tenreg} is given in Section \ref{proof:thm:tenreg}}. The first term in the right hand side (RHS) of \eqref{thm:finalbound} is due to the estimation error. In many real applications, $D=3,4$, and $R$ is very small, e.g., $R=3$ \citep{Zhou-Li-Zhu13}. Thus, $(R^D + R\sum_{d=1}^D p_d)/n $ is very close to $R\sum_{d=1}^D p_d /n$ in terms of order, which is the minimax %optimal learning rate 
lower bound shown in \cite{suzuki2015convergence}. The second term in the RHS of \eqref{thm:finalbound} is due to the approximation error. 
Note that we do not make any assumptions on the relationship between $R$ and the true rank $R_0$. That is, this theorem allows $R_0 -R$ to be arbitrarily large. If $R \ge R_0$ and $G = g(\btheta_0)$, where $\btheta_0$ is CP parameters of $\mathbf{A}_0$, we have $\delta^2=0$. 
The third term in the RHS of \eqref{thm:finalbound} is the bias due to penalization. Using a $\lambda$ satisfying $\lambda \le C\max\{(R^D + R\sum_{d=1}^D p_d)/G, n \delta^2/G\}$, the bias term is dominated by the first two terms in the RHS.

In addition to the tensor $\hat{\mathbf A}$, we are also interested in the asymptotic behavior of the corresponding CP parameters.
% when the sample size $n$ goes to infinity.
For a given set of data, Assumption \ref{assump:g} can be used to show $\hat{\btheta}$ is bounded whenever $\lambda>0$.
However, in the asymptotic framework, we are considering a sequence of data sets and $\lambda$ as $n\rightarrow\infty$. Therefore, the result is not straightforwardly followed from Corollary \ref{lem:f_lambda}.
%\ray{check}\ya{Yes. The magnitude of CP parameters is unclear as $n \to \infty$.}
Since a generic penalty function $g$ is considered in this article, we show the rate of convergence of $\hat{G}$ to obtain a general result, where $\hat{G}:=g(\hat{\btheta}) $ for an attainable $\hat{\btheta}$. 
\begin{corollary}\label{cor:hat_G}
	Using the same assumptions of Theorem \ref{thm:tenreg} and taking $$\lambda =  \max\{(R^D + R\sum_{d=1}^D p_d)/G, n \delta^2/G\}, $$ 
	we have 
	\[
	\hat{G} = \bigOp (%\frac{R^D + R\sum_{d=1}^D p_d + n\delta^2 + \lambda G}{\lambda} 
	G
	).
	\]
\end{corollary}
\supp{The proof of Corollary \ref{cor:hat_G} is in Section \ref{proof:cor:hat_G}.} 

%\subsection{A Counterexample}\label{sec:ridge}
\subsection{Not every penalty works}\label{sec:ridge}

Assumption \ref{assump:g} gives a sufficient condition of the penalty function that prevents the CP degeneracy. The crux is to control the magnitude of the estimator. In this subsection, we will give an example that the choice of parametrization of which the magnitude is being controlled is crucial.
More specifically, we will focus on the original tensor parameterization $\mathbf{A}$ instead of the CP parametrization $\btheta$.
A natural way to control its magnitude is to impose the element-wise ridge penalty ($l_2$) on the regression coefficient tensor $\mathbf{A}$
%\citep[e.g.,][]{guo2011tensor, zhang2018tensor, lock2018tensor}.
\citep[e.g.,][]{guo2011tensor}. %zhang2018tensor}.
In other words, consider
\begin{equation}\label{eq:opt_ridge}
	\argmin_{\mathrm{rank}(\mathbf A) \le R}  \sum_{i=1}^n ( y_i -  \langle \mathbf A, \mathbf X_i \rangle   )^2+\alpha \Vert \mathbf A \Vert_{F}^2.
\end{equation}

\begin{proposition}\label{lem:ridge}
	The optimization \eqref{eq:opt_ridge} is not guaranteed to have a solution. %attain the infimum, i.e., the  solution of the optimization may not be will defined.
\end{proposition}
The following Example \ref{exam:ridge} serves as a proof of Proposition \ref{lem:ridge}.
We will also compare the estimators using \eqref{eqn:def:f} and \eqref{eqn:penalized:CPlevel} in Section \ref{sec:simulation} of the simulation study.
\begin{example}\label{exam:ridge}
	%Taking $R=2$ in optimization \eqref{eq:opt_ridge}.
	If $\mathbf Z$ and $\mathbf{y}$ satisfy the conditions of Example \ref{exam:1}, then there does not exist any solution when one uses \eqref{eq:opt_ridge} with $R=2$ to estimate the low-rank coefficient tensor in tensor linear model \eqref{eqn:TLR_model_A}. 
	\supp{The proof of this result is given in Section \ref{proof_all_exam}.}
\end{example}

If the optimization \eqref{eq:opt_ridge} cannot attain the infimum, similar to Theorem \ref{thm:diverging}, the magnitude of CP parameters may diverge. This phenomenon is confirmed by our simulation study (Table \ref{table:diveringcase}).

In practice, most algorithms impose an upper bound on the number of iteration, which prevents various numerical issues due to CP degeneracy. The resulting estimators could still perform well in practice. However, the effect of such early stopping strategy requires more understanding.

\section{Numerical experiments}\label{sec:simulation}

To further understand the CP degeneracy in tensor linear regression,
we performed numerical experiments via synthetic data generated from a variety of settings. 
According to our discussion in Section \ref{sec:theory},
we compared three methods: tensor linear regression without penalization, that with $l_2$ penalty on the CP parameters $\btheta$, and that with $l_2$ penalty on the coefficient tensor $\mathbf{A}$.
All of the corresponding optimization problems were solved by a block updating algorithm \citep{guo2011tensor, Zhou-Li-Zhu13,lock2018tensor} and the details are stated in Algorithm \ref{algo} with the following notations.
Recall the definition of $f(\btheta)$ in \eqref{eqn:def:f0} and the notation $\mathbf B_d$ in \eqref{eqn:def_theta},
\[
f(\btheta) = \sum_{i=1}^n \bigg( y_i-   \bigg \langle  \sum_{r=1}^{R} \bbeta_{1,r} \circ \cdots \circ \bbeta_{D, r}, \mathbf X_i  \bigg \rangle  \bigg)^2 \quad \text{and} \quad  \mathbf B_d = (\bbeta_{d,1}, \cdots, \bbeta_{d,R}).
\]
For notational simplicity, when the tuning parameters $\lambda$ and $\alpha$ are given, we respectively rewrite the least squares method \eqref{eqn:def:f}, CP parameters penalized method \eqref{eqn:penalized:CPlevel} with $g$ specified as the ridge penalty \eqref{eqn:def:penalty:ridge}, and the coefficient tensor penalized method in \eqref{eq:opt_ridge} as 
%\begin{equation}\label{L1}
\[
L_1 (\mathbf B_1,\ldots, \mathbf B_D):=f(\btheta),
\]
%\end{equation}
%\begin{equation}\label{L2}
\[
L_2 (\mathbf B_1,\ldots, \mathbf B_D) := f(\btheta) + \lambda \sum_{d=1}^D \Vert \mathbf B_d \Vert_F^2,
\]
%\end{equation}
and
%\begin{equation}\label{L3}
\[
L_3 (\mathbf B_1,\ldots, \mathbf B_D) := f(\btheta)  + \alpha \left\Vert \sum_{r=1}^{R} \bbeta_{1,r} \circ \cdots \circ \bbeta_{D, r}  \right\Vert^2.
\]
We randomly generated the initial value $\mathbf B_d^{0}$ and then ran Algorithm \ref{algo} to get an estimate from each method. Since all of the optimizations are non-convex, we used 5 random initial values and chose the one that minimized the corresponding objective value for each method. 
\begin{algorithm}[h!]
	\caption{Block Updating Algorithm}
	\label{algo}
	\text{Initialize: } $i$, $T$, $\mathbf B_d^{0}$, $d=1,\ldots,D$. \\
	\textbf{repeat} \\
	\hspace*{0.1in} \textbf{for} $d=1,\ldots,D$ \textbf{do} \\
	\hspace*{0.25in}$	\mathbf B_d^{t+1} = \argmin_{\mathbf B_d} L_i(\mathbf B_1^{t+1},\ldots, \mathbf B_{d-1}^{t+1}, \mathbf B_d, \mathbf B_{d+1}^t \ldots, \mathbf B_D^t  )$\\
	\hspace*{0.1in} \textbf{end for} \\
	\textbf{until} $t=T$

\end{algorithm}

In this study, each simulated dataset $\{(y_i, \mathbf X_i), i=1,\ldots, n \}$ was generated from one of the following four cases including both noiseless and noisy ones:

%\[
%\begin{aligned}
%& \text{Case 1a:} 2\\
%& \text{Case 1a:} 2\\
%\end{aligned}
%\]

\begin{center}
	\begin{itemize}
		\item [Case 1a:] 	\begin{center}	
			$y = \langle \mathbf A_0^a, \mathbf X \rangle $;
		\end{center}
		\item [Case 1b:]
		\begin{center}
			$y = \langle \mathbf A_0^b, \mathbf X \rangle $;
		\end{center}
		\item [Case 2a:]
		\begin{center}	
			$y = \langle \mathbf A_0^a, \mathbf X \rangle  + \epsilon$;
		\end{center}
		\item [Case 2b:]  
		\begin{center}
			$y = \langle \mathbf A_0^b, \mathbf X \rangle + \epsilon$.
		\end{center}
	\end{itemize}
\end{center}
In the above, $\mathbf X \in \R^{p_0 \times p_0 \times p_0}$ was randomly generated with
% generated whose entry $X_{i_1i_2i_3}
independent $Unif(0,1)$ entries, and the error $\epsilon$ was an independent $\mathcal{N}(0, \sigma^2)$ random variable.
The variance $\sigma^2$ was chosen such that
the signal-to-noise ratio $\text{SNR} = \text{Var}\{m(\mathbf X)\}\big/ \sigma^2=4$, where $m$ is the regression function (i.e., $\langle \mathbf A_0^a, \mathbf X \rangle$ or $\langle \mathbf A_0^b, \mathbf X \rangle$). 
We set $R_0 =3$ and the true coefficient tensor $\mathbf A_0^a$ and $\mathbf A_0^b$ were constructed in the following ways:
\[
\mathbf A_0^a = \mathbf w_1 \circ \mathbf v_2 \circ \mathbf v_3 + \mathbf v_1 \circ \mathbf w_2 \circ \mathbf v_3+ \mathbf v_1 \circ \mathbf v_2 \circ \mathbf w_3
\quad \mbox{and} \quad
%\]
%and
%\[
\mathbf A_0^b = \mathbf u_1 \circ \mathbf u_2 \circ \mathbf u_3 + \mathbf v_1 \circ \mathbf v_2 \circ \mathbf v_3+ \mathbf w_1 \circ \mathbf w_2 \circ \mathbf w_3,
\]
where $\mathbf w_d$, $\mathbf u_d$, $\mathbf v_d \in \R^{p_0}$ with $p_0=5$, and their entries were independently drawn from $Unif(-5,5)$, $d=1,2,3$, for every simulated dataset.  
It can be seen from the data generating processes that $\mathbf w_d$, $\mathbf u_d$, $\mathbf v_d$ are linear independent, for $d=1,2,3$, with probability 1. 
Although both $\mathbf A_0^a$ and $\mathbf A_0^b$ have CP rank $3$, they behave differently in the sense of border rank. The border rank of $\mathbf A_0^a$ is strictly less than its CP rank, which is called a degenerate tensor \citep{de2008tensor}. 
We considered two different sample sizes $n=100$ and $200$ for each case.
In the experiments, Algorithm \ref{algo} was applied to the simulated datasets with $R = 2$ and $3$ for each of these three methods. We varied the tuning parameters $\lambda$ and $\alpha$, with values chosen from $\{0.001, 0.01, 0.1\}$. For each setting, the experiments were replicated 50 times.

%\subsection{Results}\label{sec:4.3results}
We investigated the CP degeneracy in the experiments by checking whether the magnitude of the CP parameters diverges as the iteration number goes to infinity.
We note that this is not a perfect rule, since the underlying optimizations are non-convex and the iteration sequences may not converge to a global optimum even with multiple initializations.
However, we will see in later results that our empirical rule of determining the CP degeneracy is fairly accurate and does not affect the overall conclusion of our experiments.

Next we construct this empirical rule.
In practice, we only have access to a finite number of iterations due to the limitation of computational resources.
In our experiments, we let the algorithm run for $100000$ iterations.
In order to numerically identify CP degeneracy, we calculated $\mathcal{M}(\btheta_t)$ to measure the magnitude of CP parameters (see also Theorem \ref{thm:diverging}), and therefore we have
% In other words, we collected
\begin{equation}\label{eqn:sim:scale_iteration_curve}
	\{t, \mathcal{M}(\btheta_t)\}, t=1,\ldots, 100000.
\end{equation} 
Two examples for Case 1a with $(n, p, R, R_0)=(200,5,2,3)$  and Case 2a with $(n, p, R, R_0)=(100,5,3,3)$ are depicted in the first row of Figure \ref{fig:2}. 
%\kejun{For which case of setting, including $n$, $R$, and $R_0$?}\ya{updated}
Determining whether the underlying iteration sequence of \eqref{eqn:sim:scale_iteration_curve} is divergent can be regarded as a classification problem. 
We propose a simple classification rule to conservatively identify the divergent curves.
First, as shown in the Figure \ref{fig:2}, the initial iterations are usually of different patterns and clearly not crucial in determining the divergence. So we focus on $t\ge 50000$. 
If $\mathcal{M}(\btheta_{100000}) \le \mathcal{M}(\btheta_{50000})$, the curve will be simply classified as non-divergent.
The difficult situations are that $\mathcal{M}(\btheta_{100000}) > \mathcal{M}(\btheta_{50000})$, where our rule is based on the estimation of the gradient of a continuous surrogate.
Consider a differentiable function defined on $[50000,\infty)$ that passes through the points $\{(t, \mathcal{M}(\btheta_t))\}_{t=50000}^{100000}$, with its gradient approximated by an exponential function 
\begin{equation}\label{eqn:def:habc}
	h_{a,b,c}(t) = at^b + c,
\end{equation}
where $a$, $b$ and $c$ are parameters crucial for determining the divergence.
These parameters are estimated by the applications of least squares estimation to the numerical proxies of the gradient:
\begin{equation}\label{eqn:sim:cutoff}
	\left\{\left(t, \frac{\mathcal{M}(\btheta_{t+100}) - \mathcal{M}(\btheta_{t})}{100}\right): t=50000, 50100\ldots, 99900\right\}. 
\end{equation}
Theoretically, when the gradient function has the form of \eqref{eqn:def:habc}, one can use the values of $a$, $b$, and $c$ to simply determine whether $\int_{50000}^{\infty} h_{{a}, {b},{c}}(t) \rm d t = \infty$ or not.
It is known that either when $\{a>0,c>0\}$ or $\{a>0, c = 0, b>-1\}$, the integral will be infinity.
Let $\hat{a}$, $\hat{b}$, and $\hat{c}$ be the fitted values of $a$, $b$, and $c$, respectively. Due to our discussion, we use the following rule
%\begin{equation}\label{eqn:sim:threshloding}
\[
\hat{a} > 0, \hat{c} > \gamma_c% \mathcal{M}(\btheta_{100000}) - \mathcal{M}(\btheta_{50000}) > 0,
\quad \text{or} \quad	\hat{a} > 0, \hat{b} \ge \gamma_b, \eta_c  \le\hat{c} \le \gamma_c
\]
%\end{equation}
%or 
%\begin{equation}
%	\hat{a} > 0, \hat{b} \ge \gamma_b, \eta_c  \le\hat{c} \le \gamma_c, %\mathcal{M}(\btheta_{100000}) - \mathcal{M}(\btheta_{50000}) > 0,
%\end{equation}
to quantitatively identify whether the magnitude curve $\mathcal{M}(\btheta_t)$ is divergent to infinity, 
where $\gamma_c$, $\eta_c >0$ and $\gamma_b > -1$ are the pre-chosen cutoffs. 
One example of such fitting is depicted in the bottom panels of Figure \ref{fig:1}. 
Since there exist estimation and approximation errors, due to fitting an exponential model, we set $(\gamma_b, \eta_c, \gamma_c)=(-0.5, 0, 0.0015)$ to obtain a conservative classification rule of divergence. 

\begin{figure}[!h]
	\centering
	%	\includegraphics[
	%	width=0.45\textwidth]{./figures/DeR3_dim555_n200_noise0_iter10_alpha_new_initial}
	\includegraphics[
	width=0.45\textwidth]{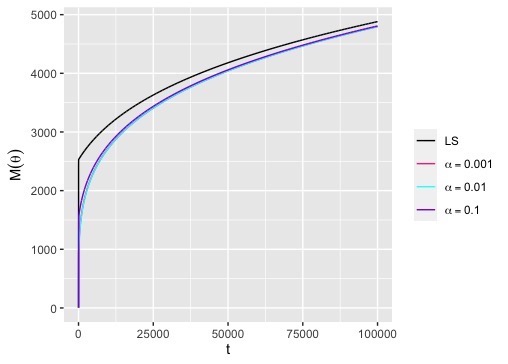}
	\includegraphics[width=0.45\textwidth]{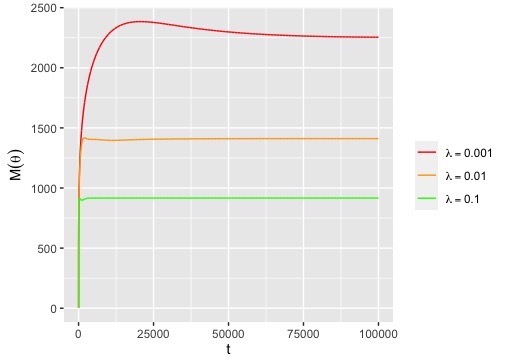}
	\includegraphics[width=0.45\textwidth]{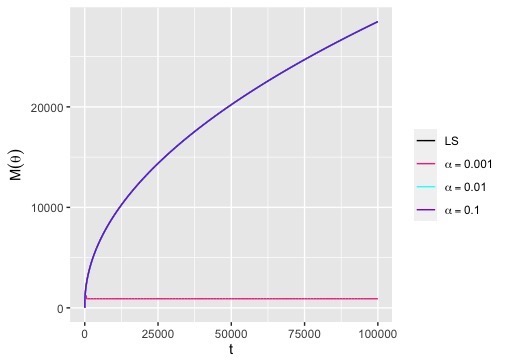}
	\includegraphics[width=0.45\textwidth]{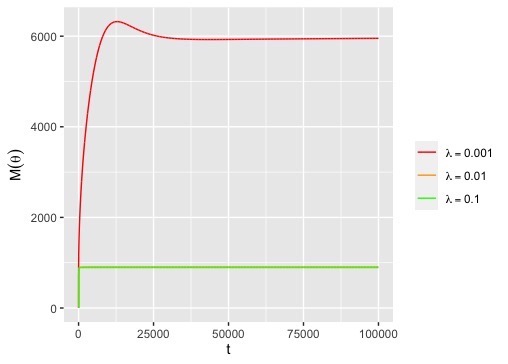}
	
	\caption{Examples of numerical experiments. The figures are the magnitude $\mathcal{M}(\btheta_t)$ versus the iteration number $t$. The top corresponds to Case 1a with $(n, p, R, R_0)=(200,5,2,3)$, and the bottom corresponds to Case 2a with $(n, p, R, R_0)=(100,5,3,3)$. 
		The results with respect to \eqref{eqn:def:f} and \eqref{eq:opt_ridge} with $\alpha = 0.001, 0.01, 0.1$ are grouped together in the first column using different colors. The results with respect to \eqref{eqn:penalized:CPlevel} with $\lambda = 0.001, 0.01, 0.1$ are depicted in the second column using different colors. \label{fig:2}.}
\end{figure}

\begin{figure}[!h]
	\centering
	\includegraphics[
	width=0.45\textwidth]{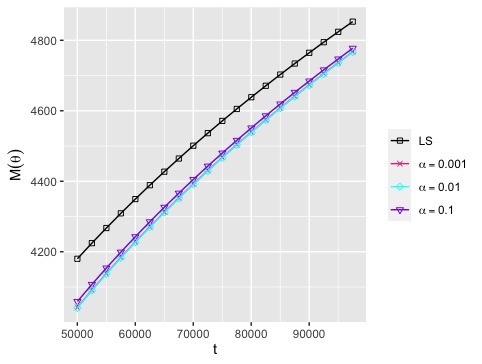}
	\includegraphics[width=0.45\textwidth]{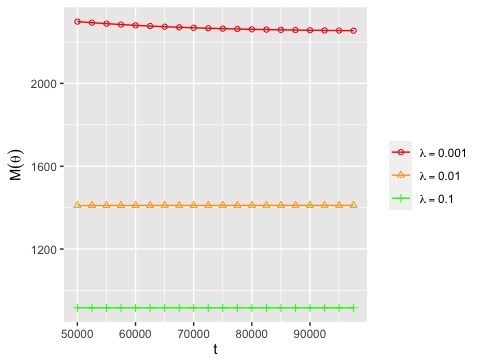}
	\includegraphics[
	width=0.45\textwidth]{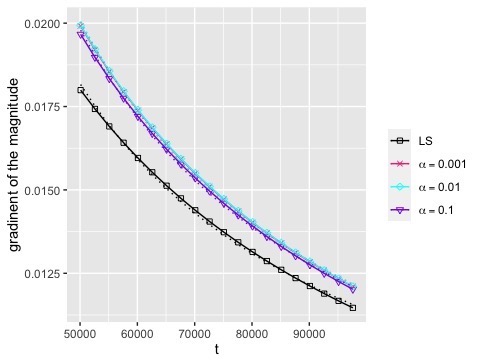}
	\includegraphics[
	width=0.45\textwidth]{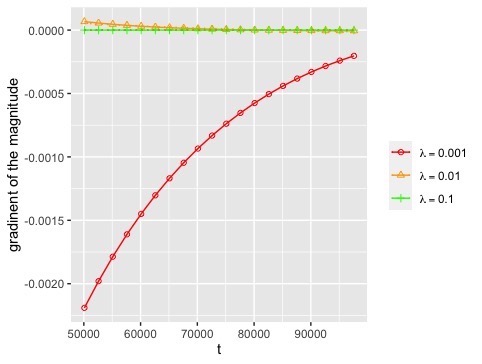}
	\caption{An example of a numerical experiment for Case 1a with $(n, p, R, R_0)=(200,5,2,3)$. The top are the plots of the magnitude $\mathcal{M}(\btheta_t)$ versus the iteration number $t$ for $t\ge 50000$. The bottom are plots of $\{(\mathcal{M}(\btheta_{t+100}) - \mathcal{M}(\btheta_{t}))/100\}$
		and the fitted values $h_{\hat{a}, \hat{b},\hat{c}}(t)$ versus $t$. 
		The results with respect to \eqref{eqn:def:f} and \eqref{eq:opt_ridge} with $\alpha = 0.001, 0.01, 0.1$ are grouped together and presented in the first column using different colors and point shapes. The results with respect to \eqref{eqn:penalized:CPlevel} with $\lambda = 0.001, 0.01, 0.1$ are depicted in the second column using different colors and point shapes. The dotted lines with the same color in the bottom are the corresponding fitted curves using \eqref{eqn:def:habc}. 	\label{fig:1}}	
\end{figure}

The results of the identified number of diverging curves using the aforementioned rule are summarized in Table \ref{table:diveringcase}. 
For the noiseless cases (Cases 1a and 1b), when the coefficient tensor is degenerate and a misspecified rank is used in optimization ($R=2$ in Case 1a), the methods \eqref{eqn:def:f} and \eqref{eq:opt_ridge} have identified %close to 
50 divergences out of 50 replications. 
This result is consistent with our discussions in Subsection \ref{ssec:existDegn}.  
For CP parameters penalized method \eqref{eqn:penalized:CPlevel} with $g$ specified as the ridge penalty \eqref{eqn:def:penalty:ridge}, %none of the iterative curves are identified as divergence in the noiseless cases \ya{There is one case that is identified as divergence }, 
there is only one case that is identified as divergence for small $\lambda$ ($\lambda = 0.001$), which confirms that CP penalization is able to overcome the problem of degeneracy as in Subsection \ref{ssec:strategy}.
For the noisy situations, i.e., Cases 2a and 2b, 
though CP parameters penalized method \eqref{eqn:penalized:CPlevel} has some rare cases to be identified as divergence, it is significantly less than the numbers of divergent cases of the other two methods. 
Again, we note that the empirical rule of divergence identification is not perfect,
and so the very small number of identified divergences does not necessarily contradict with our theory in Section \ref{ssec:strategy}.
%Furthermore, when we increase the value of tuning parameter $\lambda$ in \eqref{eqn:penalized:CPlevel}, the performance of CP parameters penalized method is improved.
%The coefficient penalized method, on the other hand, does not enjoy such improvement, when we increase the value of tuning parameter $\alpha$. 
%The coefficient penalized method has significantly more detected divergence cases.
Figure \ref{fig:1} shows an example that the least squares and the coefficient penalized methods provide divergent iteration sequences, while the results of CP parameters penalized methods are identified as convergence.
Overall, the simulation study demonstrates that the divergence in tensor regression can occur in both noiseless and noised cases, if one directly uses the least squares method or the coefficient penalized method to estimate the unknown coefficient tensor. One way to overcome this challenge is to use CP parameters penalized method as we discussed in Subsection \ref{ssec:strategy} and confirmed in our numerical experiments.

\begin{table}[h!]
	\caption{The number of curves that are identified as diverging cases. Reported are based on 50 data replications for each setting. The columns $LS$, $\lambda$'s, and $\alpha$'s correspond to the least squares method \eqref{eqn:def:f}, CP parameters penalized method \eqref{eqn:penalized:CPlevel} with $g$ specified as the ridge penalty \eqref{eqn:def:penalty:ridge}, and the coefficient penalized method \eqref{eq:opt_ridge}, respectively.
	\label{table:diveringcase} }
\resizebox{\textwidth}{!}{
	\centering
	\fbox{
		\begin{tabular*}{18cm}{ccc|cccccccl}  
			\multirow{2}*{}&  \multirow{2}*{$(n, p_0)$ }&  \multirow{2}*{$(R, R_0)$ }& \multirow{2}*{$LS$ }
			& \multirow{2}*{$\lambda=0.001$ }&
			\multirow{2}*{$\lambda=0.01$ }   & \multirow{2}*{$\lambda=0.1$ } & \multirow{2}*{$\alpha=0.001$ } & \multirow{2}*{$\alpha=0.01$ } & \multirow{2}*{$\alpha=0.1$ }\\ 
			\\
			\hline 
			\multirow{4}*{Case 1a}& (100, 5)
			& (2, 3)  & 50 &  0  &  0 & 0& 50  &  50  &  50 & \\
			& (100, 5) & (3, 3)
			& 0 &  0  &  0 & 0& 0  &  0  &  0  & 	\\
			& (200, 5) & (2, 3)
			& 50 &  1  &  0 & 0& 50  &  50  &  50  & 	\\
			&(200, 5)& (3, 3)&0 & 0  & 0 & 0 &0 &0& 0&\\
			\hline
			\multirow{4}*{Case 1b}
			
			& (100, 5) & (2, 3)
			& 0 &  0  &  0 & 0& 0  &  0  &  0  & 	\\
			& (100, 5) & (3, 3) 	& 0 &  0  &  0 & 0& 0  &  0  &  0  &  \\
			&(200, 5)& (2, 3)&0 &  0 & 0 & 0 &0 &0& 0&\\
			&(200, 5)& (3, 3)& 0 &  0 & 0 & 0 &0 &0& 0&\\
			\hline 
			\multirow{4}*{Case 2a}
			& (100, 5) & (2, 3)
			& 20 &  1  &  0 & 0& 20  &  20  &  20  & 	\\
			& (100, 5) & (3, 3)&25 & 1 & 1 & 0& 20 &18 &22 \\
			
			& (200, 5) & (2, 3)
			& 21 &  1  &  1 & 0& 21  &  21  &  21  & 	\\
			%&(200, 10)& (2, 3)&17 & 0 & 0 & 0& 17& 17& 17&\\
			&(200, 5)& (3, 3)& 25 & 1 & 1 & 0 &25 &25& 27&\\
			\hline
			\multirow{4}*{Case 2b}
			& (100, 5) & (2, 3)
			& 3 &  0  &  0 & 0& 1  &  2  &  1  & 	\\
			& (100, 5) & (3, 3)
			& 8 &  0  &  0 & 0& 9  &  6  &  9  & \\
			&(200, 5)& (2, 3)&2 &0  & 0 & 0 &2 &2& 2&\\
			&(200, 5)& (3, 3)&1 &  0& 0 & 0 &2 &1& 2&\\
	\end{tabular*}  
}
}	
\end{table}

%\section*{DATA AVAILABILITY STATEMENT}
%The data used in this work are generated by Monte Carlo simulations. $R$ code for implementing the simulations will be given if one requests.
%
%
%

%! TEX root = CPde.tex
\appendix
\setcounter{equation}{0}

\renewcommand{\theequation}{S.\arabic{equation}} 
\renewcommand{\thesection}{S.\arabic{section}}
\renewcommand{\thesubsection}{S.\arabic{section}.\arabic{subsection}}
\renewcommand{\thelemma}{{\bf S.\arabic{lemma}}} %% 
\renewcommand{\thetable}{S.\arabic{table}}
\renewcommand{\thefigure}{S.\arabic{figure}}
\section{Proof details}
\label{supp}
%\kejun{I see the notations/cross-references are not consistent with the main text. Revise first.}\ya{updated}
\subsection{Proof of the results of Examples}\label{proof_all_exam}
%\begin{proof}[Proof of Example \ref{exam:1}]
%	By definition, we have 
%	\[
%	\mathcal{S} = \{ \mathbf G_b \}.
%	\]
%	By Corollary 5.12 of \citet{de2008tensor}, $\mathrm{rank}(\mathbf G_b) =3$ and  $\mathrm{rank}_B(\mathbf G_b) =2$ . Thus, this setting satisfies the condition of Lemma  \ref{lem:divergingset}.
%\end{proof}

\begin{proof}[Proof of the result of Example \ref{exam:2}]
	%	\kejun{See my comment in the main text.}\ya{need discussion in the meeting due to large-p-and-small-n}
	To save notations, for a tensor $\mathbf G \in \mathbb{R}^{p_1\times p_2 \times p_3}$, we define
	\begin{equation}
	\label{def:tildevec}
	\begin{aligned}
	\widetilde{\vec} (\mathbf G) = (G_{1,1,1},\ldots, G_{1,1,\widetilde{p}_1}, G_{1,2,\widetilde{p}_1}, \ldots, G_{1,\widetilde{p}_2,\widetilde{p}_1}, \ldots, & \\ G_{\widetilde{p}_1,\widetilde{p}_2,\widetilde{p}_1}, \ldots, G_{\widetilde{p}_1,\widetilde{p}_2,\tilde{p}_1+1}, \ldots, &G_{\widetilde{p}_1,\widetilde{p}_2,p_3},
	\ldots, G_{p_1,p_2,p_3} )\in \mathbb{R}^{\prod_d p_d}.
	\end{aligned}
	\end{equation}
	Let 
	\begin{equation}
	\label{def:tildeZ12}
	\widetilde {\mathbf Z}= ( \mathbf Z_1 ,\mathbf Z_2 ).
	\end{equation}
	Without loss of generality, $\widetilde{\mathbf Z}$ can be written as
	\[
	\widetilde{\mathbf Z}  = (\widetilde{\vec}( \mathbf{X}_1), \ldots, \widetilde{\vec}(\mathbf X_n) )^\tp .
	\]
	Now, suppose $\mathbf G \in \mathcal{S}$, where $\mathcal{S}$ is defined as \eqref{eqn:example:3:defS}. %Note that $\vec(\cdot)$ and $\widetilde \vec_{\widetilde{p}_1,\widetilde{p}_2,\widetilde{p}_3}(\cdot)$ are just different vectorization. 
	By definition, there exists  $\mathbf b \in \mathbb{R}^{p_1p_2p_3}$ such that 
	\begin{equation}
	\label{tildeA}
	\widetilde{\vec} (\mathbf G) = (\widetilde{\mathbf Z}^\tp \widetilde{\mathbf Z})^+ \widetilde{\mathbf Z}^\tp \mathbf y + (\mathbf I - (\widetilde{\mathbf Z}^\tp \widetilde{\mathbf Z})^+ \widetilde{\mathbf Z}^\tp \widetilde{\mathbf Z}) \mathbf b.
	\end{equation}
	Recall 
	\begin{equation}\label{def:tildeAb}
	\widetilde{\mathbf G}_b = \widetilde{\mathbf v}_1 \circ \widetilde{\mathbf w}_2 \circ \widetilde{\mathbf w}_3 + \widetilde{\mathbf w}_1 \circ \widetilde{\mathbf v}_2 \circ \widetilde{\mathbf w}_3+ \widetilde{\mathbf w}_1 \circ \widetilde{\mathbf w}_2 \circ \widetilde{\mathbf v}_3,
	\end{equation}
	where $\widetilde{\mathbf w}_d,\widetilde{\mathbf v}_d \in \R^{\widetilde{p}_d}$ are pairs of linearly independent vectors,  $d=1,2,3$.
	Using \eqref{def:tildeZ12}, \eqref{tildeA} and \eqref{def:tildeAb}, we have 
	\begin{equation}
	\label{eqn:tildeA2}
	\widetilde{\vec} (\mathbf G)= \begin{pmatrix}  \vec (\widetilde {\mathbf G}_{b})  \\  	(\mathbf Z_2^\tp \mathbf Z_2)^+\mathbf Z_2^\tp\mathbf Z_2  \mathbf b_{p-s} \end{pmatrix}, 
	\end{equation}
	where $\mathbf b_{p-s} \in \mathbb{R}^{p-s}$, $s=\widetilde{p}_1\widetilde{p}_2\widetilde{p}_3$ and $\widetilde {\mathbf G}_{b} \in \mathbb{R}^{\widetilde{p}_1 \times \widetilde{p}_2 \times \widetilde{p}_3}$ is a sub-tensor of $\mathbf G$ with elements $\widetilde {\mathbf G}_{b,\widetilde{i}_1,\widetilde{i}_2,\widetilde{i}_3} =  {\mathbf G}_{\widetilde{i}_1,\widetilde{i}_2,\widetilde{i}_3} $ for $\tilde{i}_d =1,\ldots, \widetilde{p}_d$ . 
	
	Firstly, we will show $\mathrm {rank}(\mathbf G) \ge 3 $. If it does not hold, then $\mathrm {rank}(\mathbf G)\le 2$. %Assume $\mathrm {rank}(\mathbf G) \le 2$. 
	We then have the sub-tensor $\widetilde{\mathbf{G}}_b$ satisfying $\mathrm {rank}(\widetilde{\mathbf{G}}_b) \le 2$.
	However, using the definition \ref{def:tildeAb} and Corollary 5.12 of \citet{de2008tensor}, we have 
	\[
	\mathrm{rank}(\widetilde{ \mathbf G}_b) =3,
	\]
	which is a contradiction. %Therefore, $rank(\mathbf G) \neq 2$. Similarly, $rank(\mathbf G) \neq 1$. 
	Thus, $\mathrm{rank}(\mathbf G) \ge 3$. Secondly, we will show there exists a tensor in the form of \eqref{eqn:tildeA2} has border rank 2. Indeed, it can be obtained by taking  $\mathbf b_{p-s}=\mathbf 0$. %Taking $\mathbf b_{p-s}=\mathbf 0$, we will get what we want. Thus,
	In summary, this setting satisfies the condition in Lemma \ref{lem:divergingset}, which completes the proof.
	
\end{proof}

\begin{proof}[Proof of the result of Example \ref{exam:ridge}]
	The optimization  \eqref{eq:opt_ridge} can be reformulated as
	\[
	\min_{\mathrm{rank}(\mathbf A) \le R}   \Vert \mathbf y_{\alpha} - \mathbf Z_{\alpha} \vec(\mathbf A) \Vert_2^2,
	\]
	where $\mathbf y_\alpha =(\mathbf y^\tp , \mathbf 0^\tp )^\tp $  and $\mathbf Z_{\alpha }= (\mathbf Z^\tp  , \sqrt{\alpha} \mathbf I_p)^\tp $. To use Lemma \ref{lem:divergingset}, we firstly define $\widetilde{\vec}(.)^{-1} $, with is the inverse of the vectorization \ref{def:tildevec}. The set of solutions can be written as
	%to finish the proof. In this setting,
	\[
	\mathcal{S} = \bigg[ \bigg\{ \frac{1}{1+\alpha}\widetilde{\vec}^{-1} \begin{pmatrix} \vec (\widetilde{\mathbf G}_b)  \\  \mathbf 0 \end{pmatrix}   \bigg \} \bigg],
	\]
	where $\widetilde{\mathbf G}_b$ is defined as \eqref{def:tildeAb}. 
	It is straightforward to see $\mathcal{S}$ satisfies the condition of Lemma  \ref{lem:divergingset}, which completes the proof. 
\end{proof}

\subsection{Proof of Lemma \ref{lem:divergingset}}
\label{proof_of_lemma1}
\begin{proof}
	Recall that
	%	\begin{equation}
	\begin{equation}
	\label{eqn:example:3:defS1}
	\mathcal{S}=\{\mathbf A: \vec (\mathbf A) = (\mathbf Z^\tp \mathbf Z)^+ \mathbf Z^\tp \mathbf y + (\mathbf I - (\mathbf Z^\tp \mathbf Z)^+ \mathbf Z^\tp \mathbf Z) \mathbf b , \mathbf b \in \mathbb{R}^{p_1\cdots p_D}\},
	%	\end{equation}
	\end{equation}
	%	\begin{equation}
	\begin{equation}
	\label{eqn:example:3:defRm}
	R_m = \min \{\mathrm{rank}(\mathbf A): \mathbf A \in \mathcal{S} \},
	\end{equation}
	and
	\begin{equation}
	R_b = \min \{\mathrm{rank}_{\rm B}(\mathbf A): \mathbf A \in \mathcal{S} \}.
	\end{equation}
	%	\end{equation}
	Suppose $\mathbf A_m \in \mathcal{S}$, $rank(\mathbf A_m) = R_m$. % one  border rank of $\mathbf A \in \mathcal{S}$ is $R_{b}$  and 
	By definition, $R_{b} \le R < R_m $.
	Recall %the definition of $F$ %\eqref{eqn:def:F2},
	\[
	F(\mathbf A) = \Vert \mathbf y - \mathbf Z \vec ( \mathbf A ) \Vert^2.
	\]
	It is trivial to see $\inf f = \inf F$. Let $\mathbf A_t$ be the corresponding tensor of $\btheta_t$. 
	We will show two facts, i.e., 
	\begin{equation}
	\label{eqn:example:3:fact1}
	F(\mathbf A) > \inf F \quad \text{if}  \quad \mathrm{rank}(\mathbf A) \le R,
	\end{equation}
	and there exists $\{\mathbf A_t \} $ satisfying 
	\begin{equation}
	\label{eqn:example:3}
	\mathrm{rank}(\mathbf A_t) \le R \quad \text{and} \quad F(\mathbf A_t) \to \inf F. 
	\end{equation}
	The two facts, i.e., \eqref{eqn:example:3:fact1} and \eqref{eqn:example:3}, lead that the  optimization  \eqref{eqn:opt_openset} does not have a solution.

	Since $R < R_m$, it follows from \eqref{eqn:example:3:defS1} and \eqref{eqn:example:3:defRm} that \eqref{eqn:example:3:fact1} holds.
	Now, let us prove \eqref{eqn:example:3}. It is trivial to obtain 
	\begin{equation}
	\label{eqn:example:3:2}
	\begin{aligned}
	%&\inf Q \\
	%< 
	&\Vert \mathbf y - \mathbf Z \vec ({\mathbf A})  \Vert_2^2 \\
	\le & \inf F + \Vert \mathbf Z \vec ({\mathbf A_m}) - \mathbf Z \vec (\mathbf A) \Vert^2  \\
	\le & \inf F + C \Vert \mathbf A_m -  \mathbf A \Vert^2,
	\end{aligned}
	\end{equation}
	for all $\mathbf A$ satisfying $rank(\mathbf A) \le R$.
	%\[
	%R_{bm} \le rank(\mathbf A) < R_m.
	%\]
	Since $R\ge R_{bm}$, there exists a sequence $\{\mathbf A_t \}$ satisfying 
	\begin{equation}
	\label{eqn:example:3:3}
	\mathrm{rank}(\mathbf A_t) \le R \quad \text{and} \quad \Vert \mathbf A_m - \mathbf A_t \Vert \to 0. 
	\end{equation}
	Using \eqref{eqn:example:3:fact1}, \eqref{eqn:example:3:2} and \eqref{eqn:example:3:3} we finished the proof of  \eqref{eqn:example:3}. Thus, we prove the argument.
\end{proof}

\subsection{Proof of Theorem \ref{thm:diverging}}
\label{proof:thm:diverging}
\begin{proof}
	Without loss of generality, we can fix one representation of CP decomposition. Denote
	\[
	\tilde{\btheta}_t =  \{ \vec(\tilde{\mathbf B}_1^t)^\tp, \cdots, \vec(\tilde{\mathbf B}_D^t)^\tp\},
	\]
	where 
	\[
	\tilde{\mathbf B}_d^t = \bigg (  \frac{\bbeta_{d,1}^t}{\Vert \bbeta_{d,1}^t \Vert}, \ldots, \frac{\bbeta_{d,R}^t}{\Vert \bbeta_{d,R}^t \Vert} \bigg), \text{for}, \ d=1,\ldots, D-1
	\]
	and
	\[
	\tilde{\mathbf B}_D^t = \bigg (  \prod_{d=1}^{D-1} \Vert \bbeta_{d,1}^t \Vert {\bbeta_{D,1}^t}, \ldots,  \prod_{d=1}^{D-1} \Vert \bbeta_{d,R}^t \Vert {\bbeta_{D,R}^t} \bigg).
	\]
	It is straightforward that
	\begin{equation}
	\label{eqn:proof:diffrep}
	f(\btheta_t) = f(\tilde{\btheta}_t) \
	\text{and} \ \mathcal{M}(\btheta_t) = \mathcal{M}(\tilde{\btheta}_t).
	\end{equation} The flowing proof is similar to Lemma 1 of \citet{krijnen2008non}. Assume $\{ \btheta_t \}$ is a sequence such that $f(\btheta_t) \to \inf f$ and $\mathcal{M}(\btheta_t) < \infty$. Using \eqref{eqn:proof:diffrep},  $\{ \tilde{\btheta}_t \}$ is  a bounded sequence. By Bolzano-Weierstrass theorem, % subsequence $\{\btheta_{t_s}  \}$. By Bolzano–Weierstrass theorem, 
	there exists a further convergent subsequence $\{\tilde{ \btheta}_{t_{j}} \}$.
	If follows from the continuity of $f$ that 
	\[
	\lim_{j \to \infty } f(\tilde{\btheta}_{t_{j }} ) = f(\hat{\btheta}) = \inf f,
	\]
	where $\hat{\btheta} =  	\lim_{j \to \infty } \tilde{\btheta}_{t_{j }}$ and it attains the infimum of $f$. By assumptions, it leads to a contradiction. Thus, any subsequence of $\{ \tilde{\btheta}_t \}$ is unbounded, which leads to $\mathcal{M}(\btheta_t) = \mathcal{M}(\tilde{\btheta}_t) \to \infty$.
\end{proof}

\subsection{Proof of Theorem \ref{thm:nonasymptotic}}
\label{proof:thm:nonasymptotic}
\begin{proof}
	Recall
	\[
	\mathbf D_t =(\mathbf d_{1,t}, \ldots, \mathbf d_{R,t}),
	\]
	where 
	\[
	\mathbf d_{r,t} = \frac{1}{ \prod_d \Vert  \bbeta_{dr,t} \Vert  } \bbeta_{1r,t}  \otimes \cdots \otimes \bbeta_{Dr,t} \in \mathbb{R}^{\prod p_d}.
	\]
	There exists a vectorization rule $\bar{vec}(\cdot)$ for a tensor, such that 
	\begin{equation}
	\bar{\vec}(\mathbf A_t) = \mathbf D_t \mathbf h_t,
	\end{equation}
	where $\mathbf A_t) $ is the corresponding tensor of $\btheta_t$ and
	\begin{equation}
	\label{def:h}
	\mathbf{h}_t = \bigg(\prod_d \Vert  \bbeta_{d1,t} \Vert, \cdots, \prod_d \Vert  \bbeta_{dR,t} \Vert \bigg)^\tp  \in \mathbb{R}^{R}.
	\end{equation}
	For simplicity, we denote 
	\begin{equation}
	\label{eqn:def:barZ}
	\bar{\mathbf Z} = (\bar{\vec}( \mathbf X_1)^\tp , \ldots, \bar{\vec}(\mathbf X_n)^\tp )^\tp.
	\end{equation}
	Using the specific rearrange rule \eqref{eqn:def:barZ}, 
	\eqref{eqn:def:F} at iteration $t$ can be written as 
	%	 \[
	%	 \min_{\btheta} \Vert \mathbf y - \bar{\mathbf Z} \mathbf D \mathbf h \Vert_2^2.
	%	 \]
	\[
	F(\mathbf A_t) = \Vert \mathbf y - \bar{\mathbf Z} \mathbf D_t \mathbf h_t  \Vert^2.
	\]
	We will see the performance of $\mathbf D_t$ when the infimum does not attain as follow.
	
	Firstly, we note that there exits a constant $t_0$, if $t \ge t_0$, then
	\begin{equation}
	\label{Dthtbounded}
	\Vert \bar{\mathbf Z} \mathbf  D_t \mathbf h_t \Vert = 	\Vert \bar{\mathbf Z} \mathbf  D_t \mathbf h_t - \mathbf y + \mathbf y \Vert \le \Vert \bar{\mathbf Z} \mathbf  D_t \mathbf h_t - \mathbf y\Vert + \Vert \mathbf y \Vert <  1 + \Vert \mathbf y \Vert
	\end{equation}
	
	Secondly, we will show the final result. 	Using Assumption \ref{assump:deterministic_X}, we have
	\[
	\Vert  \mathbf D_t \mathbf h_t \Vert_2^2 \le \frac{C\Vert \mathbf Z \mathbf D_t \mathbf h_t \Vert_2^2}{n}. 
	\]
	By definition, we have
	\[
	\lambda_{\min}(\mathbf D_t^\tp \mathbf D_t) \Vert  \mathbf h_t \Vert_2^2 \le	\Vert  \mathbf D_t \mathbf h_t \Vert_2^2 . 
	\]
	Therefore,
	\[
	\lambda_{\min}(\mathbf D_t^\tp \mathbf D_t) \le \frac{C_1 + C_2 \Vert \mathbf y \Vert^2 }{n\Vert \mathbf h_t \Vert^2 }.
	\]
	The proof about  $\mathbf D_t$ can be finished by noting $\mathcal{M}(\btheta_t) = \Vert  \mathbf h_t \Vert_1 \le \sqrt{R} \Vert  \mathbf h_t \Vert$, where $\Vert . \Vert_1$  is the $l_1$-norm of a vector.
	The proof of $\mathbf B_{d,t}$ can be finished by using the proof of
	Corollary 2 in \citet{krijnen2008non}.

\end{proof}

\subsection{Proof of Theorem \ref{thm:conlinar:asym}}
\label{proof:thm:conlinar:asym}
We first present the argument about restricted eigenvalue.
\begin{lemma}
	\label{lem:eigen}
	Under Assumptions \ref{assump:predictor} and \ref{assump:error}, if $n > C w^2(\mathcal{P})$, we have 
	\[
	C_1 \le \inf_{\mathbf A \in \mathcal{P}} \frac{1}{n} \bigg \vert  \sum_{i=1}^n \langle \mathbf A, \mathbf X_i \rangle \bigg \vert^2 \le \sup_{\mathbf A \in \mathcal{P}} \frac{1}{n} \bigg \vert  \sum_{i=1}^n \langle \mathbf A, \mathbf X_i \rangle \bigg \vert^2 \le C_2,
	\]
	with  probability at least $1 - 2 \exp(-C_3 w^2(\mathcal{P}))$. Further, $w(\mathcal{P})$ in the lemma can be replace to be $V$ with $V \ge w(\mathcal{P})$.
\end{lemma}

\begin{proof}[Proof of  Lemma \ref{lem:eigen}] 
	
	Using Theorem 10 and 12 of \citet{banerjee2015estimation} and Assumption \ref{assump:predictor} , we will finish the proof.
\end{proof}
Now we turn to the proof of  Theorem \ref{thm:conlinar:asym}.
\begin{proof}
	Using Lemma \ref{lem:eigen} and the same arguments in the proof of Theorem \ref{thm:nonasymptotic}, we obtain
	\[
	Cn \lambda_{\min}(\mathbf D_t^\tp \mathbf D_t) \Vert \mathbf h_t \Vert^2  \le 	\Vert \mathbf Z \mathbf D_t \mathbf h_t \Vert_2^2 ,
	\]
	with probability at least $1 - 2 \exp\{-C_3 w^2(\mathcal{P})\}$, which implies 
	\[
	\lambda_{\min}(\mathbf D_t^\tp \mathbf D_t) \le \frac{C}{n\Vert \mathbf h_t \Vert^2 },
	\]
	with probability at least $1 - 2 \exp\{-C_3 w^2(\mathcal{P})\}$. We then will finish the proof by following the arguments in the proof of Theorem \ref{thm:nonasymptotic}.
\end{proof}

\subsection{Proof of Theorem \ref{thm:gaussanwidthsbound}}
\label{proof:gaussanwidthsbound}
\begin{proof}
	The proof of Theorem \ref{thm:gaussanwidthsbound} can be finished by using Lemma \ref{lem:gaussian_width} and \ref{lem:covering_number}.
\end{proof}
\begin{lemma}
	\label{lem:gaussian_width}
	We have 
	\[
	w(\mathcal{P}) \le  C \int_0^2 \{ \log N(\xi,\mathcal{P}, l_2)\}^{1/2} \mathrm{d}\xi.
	\]
\end{lemma}	
\begin{proof}[Proof of  Lemma \ref{lem:gaussian_width}] 
	
	This is a direct result of Dudley’s integral entropy bound \citep{koltchinskii2011oracle,vershynin2018high}
\end{proof}
\begin{lemma}
	\label{lem:covering_number}
	Suppose $N(\xi, \mathcal{P}, l_2 )$ is the covering number of $\mathcal{P}$. We then have 
	\begin{equation}
	\label{eqn:convering_number}
	N(\xi, \mathcal{P}, l_2)  \le  \big({C}/{\xi} \big)^{R^D + R\sum_{d=1}^Dp_d}.
	\end{equation}
\end{lemma}
\begin{proof}[Proof of  Lemma \ref{lem:covering_number}] 
	Note that CP decomposition is a special Tucker Decomposition, which will lead a higher-order singular value decomposition (HOSVD)\citep[see, e.g.,][]{de2000multilinear}. By 
	Lemma 2 of \citet{rauhut2017low}, we obtain \eqref{eqn:convering_number}.
\end{proof}

\subsection{Proof of Corollary \ref{lem:f_lambda}}
\label{proof:lem:f_lambda}
\begin{proof}
	Let 
	\[
	f_{\lambda}(\btheta)  =  f(\btheta) + \lambda  g(\btheta).
	\]
	Assume the optimization \eqref{eqn:penalized:CPlevel} does not have a solution, which yields that $f_{\lambda}$  can not attain the infimum. By definition, there exists a sequence $\{\btheta_{\lambda, t} \}_t$ satisfying
	\begin{equation}
	\label{proof:lem:lambda}
	f_\lambda(\btheta_{\lambda, t} ) \to \inf f_{\lambda}
	\end{equation} as $t \to \infty$. 
	
	If $	\mathcal{M}(\btheta_{\lambda,t}) < \infty$, then there exits a constant $S_1 < \infty$ such that $	\mathcal{M}(\btheta_{\lambda,t})  \le S_1$. By %the conditions on $g(\cdot)$,
	Assumption \ref{assump:g} and \eqref{proof:lem:lambda}, 
	we have a bounded sequence $\{\tilde{\btheta}_{\lambda,t}\}$ satisfying $\Vert \tilde{\btheta}_{\lambda,t}\Vert  \le S_2$ and  $f_\lambda(\tilde{\btheta}_{\lambda, t} ) \to \inf f_{\lambda}$. Using the arguments of Theorem \ref{thm:diverging}, the infimum of $f_{\lambda}$ is attainable, which contradicts the assumption. Now, we continue the proof. If
	\[
	\mathcal{M}(\btheta_{\lambda,t}) \to \infty,
	\]
	using Assumption \ref{assump:g} %the conditions on $g$ 
	, we then have 
	\begin{equation}
	\label{proof:lem:lambda2}
	f_{\lambda}(\btheta_{\lambda, t}) = f(\btheta_{\lambda,t}) + \lambda g(\btheta_{\lambda,t})  \ge  \lambda g(\btheta_{\lambda,t}) \to \infty.
	\end{equation}
	However,
	\[
	\inf f_{\lambda} \le f_{\lambda}(\mathbf 0) = f(\mathbf 0 ) + \lambda g(\mathbf 0 ) =\Vert \mathbf y \Vert^2 +\lambda g(\mathbf 0 ) < \infty.
	\]
	which contradict to \eqref{proof:lem:lambda} and \eqref{proof:lem:lambda2}. 
	Thus, we have finished the proof.

\end{proof}

\subsection{Proof of Theorem \ref{thm:tenreg}}
\label{proof:thm:tenreg}
Suppose
\[
\mathbf A_m \in \argmin_{\stackrel{\rank(\mathbf A) \le R,}{ g(\btheta_A) \le G}} \Vert \mathbf A  - \mathbf A_0 \Vert_{HS}^2.
\]
By definition 
\[
\sum_{i=1}^n (y_i - \langle \hat{\mathbf A}, \mathbf X_i \rangle )^2 \le  \sum_{i=1}^n (y_i - \langle \mathbf A_m, \mathbf X_i \rangle )^2 + \lambda g(\mathbf \btheta_{A_m} ),
\]
where $ \btheta_{A_m} $ is one CP parameterization of $\mathbf A_m$.
We then have
\[
\Vert \mathbf y - \mathbf Z  \vec(\hat{\mathbf A}) \Vert_2^2 	 \le \Vert \mathbf y - \mathbf Z \mathbf \vec (\mathbf{A}_m) \Vert_2^2  + \lambda {G}.
\]
%where $\hat{\mathbf a} = \vec(\hat{\mathbf A})$ and ${\mathbf \vec (\mathbf{A}_m)} = \vec({\mathbf A}_m)$. 
It follows that
\begin{equation}
\label{eqn:thm:all}
\begin{aligned}
\Vert \mathbf Z({ \vec ( \hat{\mathbf A} ) } - \mathbf \vec (\mathbf{A}_m)) \Vert_2^2 \le  2 \langle  \mathbf{\epsilon}, \mathbf Z \vec ( \hat{\mathbf A} ) - \mathbf Z \mathbf \vec (\mathbf{A}_m) \rangle + & \\
2 \langle \mathbf Z \mathbf r_{0m} , \mathbf Z \vec ( \hat{\mathbf A} ) - & \mathbf Z \mathbf \vec (\mathbf{A}_m) \rangle  +  \lambda G,
\end{aligned}
\end{equation}
where $\mathbf r_{0m} = \vec(\mathbf A_0) - \vec (\mathbf{A}_m)$ and $\mathbf{\epsilon}=(\epsilon_1, \ldots, \epsilon_n)^\tp$. 

Firstly, we find the upper bound of the first term on the right hand side of \eqref{eqn:thm:all}. Using the Dudley’s integral entropy bound \citep{koltchinskii2011oracle,vershynin2018high}, Assumption \ref{assump:error}, and Lemmas \ref{lem:eigen} -- \ref{lem:covering_number}, 
we will obtain
\begin{equation}
\label{eqn:thm:2edterm}
\langle  \mathbf {\epsilon}, \mathbf Z \vec(\hat{\mathbf{A}}) - \mathbf Z \mathbf \vec (\mathbf{A}_m) \rangle  \le C \sqrt{n} 
\Vert \vec(\hat{\mathbf{A}}) -  \vec(\mathbf A_m) \Vert \bigg(R^D + R\sum_{d=1}^D p_d \bigg )^{1/2},	
% \int_0^2 \{ \log N(\xi,\mathcal{P}, l_2)\}^{1/2} d\xi,
\end{equation}
with probability at least 
\[
%1 - C_1 \exp\{-C_2 w^2(\mathcal{P})\},
1 - C_1 \exp \bigg\{-C_2 \bigg(R^D + R\sum_{d=1}^D p_d \bigg) \bigg \}.
\]
Secondly, we obtain the upper bound of the second term on the right hand side of \eqref{eqn:thm:all}. Using Cauchy-Schwarz inequality yields
\begin{equation}
\label{eqn:thm:1stterm}
\langle \mathbf Z \mathbf r_{0m} , \mathbf Z \vec(\hat{\mathbf{A}})- \mathbf Z \mathbf \vec (\mathbf{A}_m) \rangle \le \Vert \mathbf Z \mathbf r_{0m} \Vert \Vert  \mathbf Z \vec(\hat{\mathbf{A}})- \mathbf Z \mathbf \vec (\mathbf{A}_m) \Vert.
\end{equation}
We will assume $n \ge C (R^D + R\sum_{d=1}^D p_d )$ in the following proof.
%########
By Assumptions, $\vec( \mathbf{X}_i)$ is a sub-Gaussian random vector and 
\[
\Vert \vec( \mathbf{X}_i) \Vert_{\psi_2} \le C.
\]
Suppose $\Vert \mathbf{r}_{0m} \Vert \ne 0$,  we then have
\[
\Vert \vec( \mathbf{X}_i)^\tp \mathbf{r}_{0m} \Vert_{\psi_2}/\Vert \mathbf r_{0m} \Vert \le C. 
\]
%$\vec( \mathbf{X}_i)^\tp \mathbf{r}_{0m}$ is sub-Gaussian. 
By Lemma 2.7.6 of \cite{vershynin2018high}, $\{\vec( \mathbf{X}_i)^\tp \mathbf{r}_{0m}\}^2/\Vert \mathbf{r}_{0m} \Vert^2$ is sub-exponential and 
\[
\Vert \{\vec( \mathbf{X}_i)^\tp \mathbf{r}_{0m}\}^2/\Vert \mathbf{r}_{0m} \Vert^2 \Vert_{\psi_1} \le C,
\]
where $\Vert \cdot \Vert_{\psi_1}$ is the sub-exponential norm \citep{vershynin2018high}.  Using Proposition 2.7.1  of \cite{vershynin2018high}, we have 
\[
\bigg\Vert  \frac{1}{n}\Vert \mathbf Z \mathbf r_{0m} \Vert_2^2/ \Vert \mathbf{r}_{0m} \Vert^2 \bigg\Vert_{\psi_1} \le C,
\]
which leads to
\[
\frac{1}{n}\Vert \mathbf Z \mathbf r_{0m} \Vert_2^2/ \Vert \mathbf{r}_{0m} \Vert^2  \le v,
\]
with probability at least $1-2\exp(-Cv)$ for $v > 0$.  Thus 
\begin{equation}
\label{proof:thm:Zr2}
\frac{1}{n}\Vert \mathbf Z \mathbf r_{0m} \Vert_2^2   \le v \Vert \mathbf{r}_{0m} \Vert^2,
\end{equation}
with probability at least $1-2\exp(-Cv)$ for all $v \ge 0$. 
Using Lemmas  \ref{lem:eigen} -- \ref{lem:covering_number}, we have 
\begin{equation}
\label{eqn:thm:upperandlower}
C_1 \sqrt{n}  \Vert  \vec(\hat{\mathbf{A}}) -  \mathbf \vec (\mathbf{A}_m) \Vert_2 \le \Vert  \mathbf Z \vec(\hat{\mathbf{A}})- \mathbf Z \mathbf \vec (\mathbf{A}_m) \Vert_2  \le C_2\sqrt{n }  \Vert  \vec(\hat{\mathbf{A}}) -  \mathbf \vec (\mathbf{A}_m) \Vert_2,
\end{equation}
with probability at least $1-C_1\exp \{-C_2(R^D + R \sum_{d=1}^Dp_d) \}$. 
Applying \eqref{eqn:thm:2edterm}, \eqref{eqn:thm:1stterm},
\eqref{proof:thm:Zr2} and \eqref{eqn:thm:upperandlower} to \eqref{eqn:thm:all}, we have
\[
\begin{aligned}
\Vert \vec(\hat{\mathbf{A}}) -\mathbf \vec (\mathbf{A}_m) \Vert_2^2 \le C \delta \sqrt{ v}   \Vert  \vec(\hat{\mathbf{A}}) -  \mathbf \vec (\mathbf{A}_m) \Vert_2 + & \\  C  \Vert  \vec(\hat{\mathbf{A}}) - & \mathbf \vec (\mathbf{A}_m) \Vert_2  \sqrt{\frac{R^D + \sum_{d=1}^D Rp_d}{n}} + \lambda G,
\end{aligned}
\]
with probability at least 
\begin{equation}
\label{eqn:thm:finalprob}
1 - C_1 \exp \bigg\{-C_2 \bigg(R^D + R\sum_{d=1}^D p_d \bigg) \bigg\}  - C_3\exp \{ -C_4v\}.% 2\exp \{ - C \min (v^2, v ) n \}
%2\exp \bigg\{ - C \min \bigg(\frac{v^2}{\kappa^2}, \frac{v}{\kappa}\bigg) n \bigg \}. 
\end{equation}
Solving the inequality yields, 
\[
\Vert \vec(\hat{\mathbf{A}}) -\vec(\mathbf A_m) \Vert_2^2 \le C_1 \frac{R^D + R\sum_{d=1}^D p_d }{n} + C_2 v \delta^2 + \lambda C_3 \frac{ G}{n},
\]
which implies
\[
\Vert \vec(\hat{\mathbf{A}}) - \vec(\mathbf A_0) \Vert_2^2 \le C_1 \frac{R^D + R\sum_{d=1}^D p_d }{n} + (C_2v +C)\delta^2+ \lambda C_3 \frac{G}{n},
\]
with probability at least \eqref{eqn:thm:finalprob}. We then finish the proof.
%\end{proof}

\subsection{Proof of Corollary \ref{cor:hat_G}}
\label{proof:cor:hat_G}
By definition, $G$ in \eqref{eqn:thm:all} can be replaced by $G - \hat{G}$. Using the same arguments, have 
\[
\Vert \vec(\hat{\mathbf{A}}) - \vec(\mathbf A_0) \Vert_2^2 \le C_1 \frac{R^D + R\sum_{d=1}^D p_d }{n} + (C_2v + C) \delta^2+ \lambda C_3 \frac{G-\hat{G}}{n},
\]
with probability at least \eqref{eqn:thm:finalprob}.
Note that $\Vert \vec(\hat{\mathbf{A}}) - \vec(\mathbf A_0) \Vert_2^2 \ge 0$. We then have
\[
\lambda C \hat{G} \le R^D + R\sum_{d=1}^D p_d  + (C_2v+C) n \delta^2 + \lambda C_3 G,
\]
with probability at least \eqref{eqn:thm:finalprob}, which completes the proof by taking $$\lambda =  \max\{(R^D + R\sum_{d=1}^D p_d)/G, n \delta^2/G\}.$$

\bibliographystyle{rss} %Dr. Wong
\bibliography{reference}

\end{document}